\documentclass[11pt, a4paper]{article}

\usepackage[margin=1in]{geometry}
\usepackage{url}
\usepackage[hidelinks]{hyperref}
\usepackage{graphicx}
\usepackage{amsmath}
\usepackage{amsthm}
\usepackage{algorithm}
\usepackage{algorithmic}

\newtheorem{theorem}{Theorem}

\usepackage{mathtools}

\usepackage{amsfonts}
\usepackage{amsmath}
\usepackage{amssymb}
\usepackage{amsthm}
\usepackage{color}
\usepackage{graphicx}
\usepackage{mathrsfs} 

\allowdisplaybreaks

\def\ttt{\texttt}

\def\extraspacing{\vspace{3mm} \noindent}
\def\figcapup{\vspace{-1mm}}
\def\figcapdown{\vspace{-0mm}}

\def\vgap{\vspace{1mm}}

\def\tabpos{\hspace{4mm} \= \hspace{4mm} \= \hspace{4mm} \= \hspace{4mm} \= \hspace{4mm} \= \hspace{4mm} \= \hspace{4mm} \= \hspace{4mm} \= \hspace{4mm} \= \hspace{4mm} \= \hspace{4mm} \= \hspace{4mm} \= \hspace{4mm} \= \hspace{4mm} 
\kill}
\newcommand{\mytab}[1]{\begin{tabbing}\tabpos #1\end{tabbing}}

\newtheorem{lemma}[theorem]{Lemma}

\newcommand{\boxminipg}[2]{\begin{center}\fbox{\begin{minipage}{#1}#2\end{minipage}}\end{center}}
\newcommand{\minipg}[2]{\begin{center}\begin{minipage}{#1}#2\end{minipage}\end{center}}
\newcommand{\myitems}[1]{\begin{itemize} #1 \end{itemize}}

\newcommand{\bm}[1]{\textrm{\boldmath${#1}$}}

\newcommand{\myeqn}[1]{\begin{eqnarray}#1\end{eqnarray}}
\newcommand{\set}[1]{\{#1\}}

\newcommand{\explain}[1]{(\textrm{#1})}

\def\eps{\epsilon}
\def\fr{\frac}
\def\-{\mbox{-}}

\def\real{\mathbb{R}}

\def\tO{\tilde{O}}

\def\nn{\nonumber}

\def\Pr{\mathbf{Pr}}

\DeclareMathOperator*{\poly}{poly}
\DeclareMathOperator*\expt{\mathbf{E}}

\allowdisplaybreaks

\def\A{\mathcal{A}}

\def\F{\mathcal{F}}

\def\X{\mathcal{X}}

\DeclarePairedDelimiter{\norm}{\lVert}{\rVert}
\newcommand{\dotpair}[2]{\left\langle #1 , #2 \right\rangle}

\newcommand{\ora}{\ttt{oracle}}

\def\vgap{\vspace{1mm}}
\def\vslit{\vspace{0.5mm}}
\def\figcapup{\vspace{-3mm}}
\def\figcapdown{\vspace{-3mm}}
\def\extraspacing{\vspace{2mm} \noindent}

\title{Distributed Learning with Adversarial Gradient Perturbations}

\author{
    Nawapon Sangsiri \\[2mm]
    CUHK \\
    {\em sangsiri@cse.cuhk.edu.hk}
    \and Yufei Tao \\[2mm]
    CUHK \\
    {\em taoyf@cse.cuhk.edu.hk}
}

\begin{document}

\maketitle

\begin{abstract}
    Privacy concerns in distributed learning often lead clients to return intentionally altered gradient information. We consider the problem of learning convex and $L$-smooth functions under {\em adversarial gradient perturbation}, where a client's gradient reply to a server query can deviate arbitrarily from the true gradient subject to a distance bound. Our study focuses on two fundamental questions: (i) what is the smallest achievable {\em sub-optimality gap} (i.e., excess error in optimization) under such responses, and (ii) how many queries are sufficient to guarantee a given sub-optimality gap? We establish tight feasibility thresholds on the sub-optimality gap and provide algorithms that achieve these thresholds with provable query complexity guarantees.
\end{abstract}

\thispagestyle{empty}
\pagebreak

\setcounter{page}{1}

\section{Introduction} \label{sec:intro}

Modern learning systems increasingly rely on edge-based architectures, where local devices compute gradient information and communicate it to a central server. In such environments, local gradient reports may be deliberately perturbed to protect sensitive data. This motivates the study of learning algorithms that operate under {\em adversarially corrupted gradients}, where each individual gradient reply may deviate arbitrarily from the true gradient, subject only to a known bound on the perturbation magnitude. Understanding what can and cannot be learned in this setting, as well as the query complexity required to do so, is a fundamental question.

\vgap

Formally, we consider a setting where there are $n$ geographically distributed clients, each holding a convex and $L$-smooth {\em loss function} $\ell_i: \real^d \rightarrow \real$ where $i \in [n]$.\footnote{$[x]$ represents the set $\set{1, 2, ..., x}$ for an integer $x \ge 1$.} Define
\myeqn{
    f(w) &=& \fr{1}{n} \sum_{i=1}^n \ell_i(w). \label{eqn:intro:f-loss-sum}
}
which is minimized at some $w^* \in \real^d$ --- when multiple minimizers exist, we define $w^*$ to be one with the smallest $\norm{w^*}$.\footnote{All norms in this paper refer to the $L_2$ norm.}
The {\em sub-optimality gap} of a vector $w \in \real^d$ is defined as the difference $f(w) - f(w^*)$. A server's goal is to find a vector whose sub-optimality gap is bounded by $\tau > 0$.

\vgap

Communication between the server and clients is limited to the following query primitive:

\minipg{0.85\linewidth}{
    {\bf Approximate Gradient Query:} The server selects an $i \in [n]$ and sends the $i$-th client a vector $w \in \real^d$. The client replies with a vector $v$ satisfying
    \myeqn{
        \norm{v - \nabla \ell_i(w)} \le \eps \label{eqn:intro:qry}
    }
    where $\eps$ is a value agreed upon in advance by the server and all clients.
}

\noindent Crucially, the client may respond with an {\em arbitrary} vector $v$ satisfying \eqref{eqn:intro:qry}, including one chosen adversarially so as to impede the server's optimization procedure. Larger values of $\eps$ expand the set of permissible replies and thus offer stronger privacy protection on the local data contributing to the computation of $\nabla \ell_i (w)$.

\vgap

We refer to the above problem as {\em distributed learning with adversarial gradient perturbations} (DLAGP). Our goal is to characterize the fundamental feasibility of learning and derive interesting query bounds in this setting. In particular, we aim to answer two questions:
\myitems{
    \item {{\bf Q1}:} How small can the sub-optimality gap be made? Formally, what is the smallest value $\tau_{min}$ such that, for all legal inputs $\set{\ell_1, \ell_2, ..., \ell_n}$, it is possible for the server to find a vector $w$ satisfying $f(w) \le f(w^*) + \tau_{min}$, possibly using infinitely many queries?

    \vslit

    \item {{\bf Q2}:} Given a value $\tau$ that is not ``unreasonably small'', how many approximate gradient queries are needed for the server to guarantee a sub-optimality gap at most $\tau$?
}


\extraspacing {\bf Our Results and Implications.} We establish non-trivial answers to both questions above.
\myitems{
    \item {Answers for {\bf Q1} ---} If no upper bound on $\norm{w^*}$ is known in advance, then no algorithm, even with infinitely many queries, can guarantee finding a vector $w$ with a bounded sub-optimality gap, regardless of the target value $\tau$.

    \vgap

    Motivated by this impossibility result, we consider {\em $R$-instances} of the DLAGP problem, where a bound $\norm{w^*} \le R$ is known. For this class of instances, we show that the smallest achievable sub-optimality gap $\tau_{min}$ lies in the interval $[\fr{1}{2}\eps R, 5\eps R]$. In other words, when $\tau < \eps R/2$, no algorithm can guarantee a sub-optimality gap of $\tau$ over all $R$-instances, even with infinitely many queries. Conversely, the server can always achieve a sub-optimality gap at most $\tau$ for all $R$-instances whenever $\tau \ge 5\eps R$.

    \item {Answers for {\bf Q2} ---} For any $R$-instance satisfying the condition $\tau \ge 5\eps R$, the server can deterministically ensure a sub-optimality gap at most $\tau$ using
    \myeqn{
        O(n \cdot \min\{LR^2 / \tau, LR/ \eps\})
        \label{eqn:intro:q2-bound1}
    }
    queries. 

    For any $R$-instance satisfying the condition $\tau \ge 5.01 \cdot \eps R$ (the constant 5.01 can be replaced with any value larger than 5), the server can, with probability at least $1-\delta$, ensure a sub-optimality gap at most $\tau$ using
    \myeqn{
        \tO
        \left(
            \fr{L R^4 (B_0 + LR)^2}{\tau^3}
        \right)
        \label{eqn:intro:q2-bound2}
    }
    queries, where
    \myeqn{
        B_0 = \max_{i=1}^n \norm{\nabla \ell_i(\bm{0})}.
        \label{eqn:intro:B0}
    }
    and $\tO(.)$ hides factors logarithmic to $L, R, 1/\tau$, and $1/\delta$; note that \eqref{eqn:intro:q2-bound2} does not depend on $n$ or $d$. All queries in this algorithm select clients uniformly at random. By standard sampling arguments, when $n$ is sufficiently large, with high probability each client is queried at most once, and most clients are never queried.
}

%
\section{Related Work} \label{sec:related}

In the study of {\em oracle complexity} in convex optimization --- an area comprehensively surveyed by \cite{b15} --- the goal is to minimize a convex function $f$ $:$ $\real^d$ $\rightarrow$ $\real$ by making the fewest calls to an ``oracle''. Given a $w \in \real^d$, the oracle returns an exact or approximate derivative of $f$ of a specific order at $w$. For example, an exact zeroth-order oracle returns $f(w)$, while a first-order counterpart returns $\nabla f(w)$. Our study contributes to the research on first-order oracles, of which see representative works \cite{abrw12,bdl23,cwt+23,d08b,dbsx12,dgn14,gg23,hhl+21,hl24,j13,kl17,kmjb24,l20,lt21b,mbl20,njls09,ny83,n12,t03} and the references therein.

\vgap

In the traditional studies of ``inexact gradient methods'', an approximate first-order oracle returns a stochastically perturbed version of $\nabla f(w)$, typically a random $\ora(w) \in \real^d$ whose expectation is $\nabla f(w)$. In some cases, there may be a ``concentration property'' to limit how often $\ora(w)$ can deviate from $\nabla f(w)$ significantly. The seminal work of \cite{dgn14} pioneered the concept of non-stochastic inexact first-order oracles with bounded errors in smooth convex optimization. Our permutation model --- formalized as approximate gradient queries (see \eqref{eqn:intro:qry}) --- aligns with their concept. Indeed, those queries promise neither expectation nor concentration guarantees, and their answers can be adversarial subject to bounded deviation. This model has gained attention in recent years \cite{cnw+22,izy22,wt24,wt25} --- we will discuss the findings of these papers in Section~\ref{sec:apg-opt:discussion}.

\vgap

A parallel line of work studies Byzantine-resilient distributed optimization; see, e.g., \cite{aal18,bmgs17,khj21,ycrb18} and the references therein. These works typically assume that most clients return exact or stochastically perturbed gradients, while a fraction of ``Byzantine'' clients may behave arbitrarily. Algorithms in this literature crucially depend on the existence of sufficiently many honest clients and exploit redundancy through robust aggregation schemes. In contrast, our model does not posit any well-behaved clients: every gradient query may be adversarially perturbed, and no stochastic population-level structure is available. As a result, classical Byzantine-robust techniques do not apply in our scenarios.

\vgap

Loosely speaking, our work is also related to {\em federated learning}, which aims to train models across geographically distributed clients while allowing them to keep data localized; see \cite{lsts20} for an excellent introduction. A central focus in that area, however, is handling heterogeneous data distributions and ensuring convergence despite differences across client datasets. In contrast, our study abstracts away data heterogeneity and instead examines fundamental limits on learnability, as well as query complexity, under worst-case, bounded adversarial gradient perturbations.

\section{Foundation: The Theory with One Client} \label{sec:agp-opt}

This section considers DLAGP with only $n = 1$ client. As we will see Section~\ref{sec:dl}, the theory developed in this section plays a vital role in solving the problem in its general form.

\vgap

When $n = 1$, DLAGP can be rephrased in a more succinct manner. Let $f: \real^d \rightarrow R$ be a function such that:
\myitems{
    \item {{\bf C1:}} $f$ is convex and $L$-smooth;
    \item {{\bf C2:}} $f$ is minimized at some $w^* \in \real^d$; if multiple minimizers exist, define $w^*$ to be one with the smallest norm.
}
The goal of optimization is to find a vector $w \in \real^d$ whose sub-optimality gap $f(w) - f(w^*)$ is bounded by a given $\tau$. The only way for a learner algorithm to glean information about $f$ is to interact with an {\em adversarial gradient perturbation} (AGP) oracle. Given a $w \in \real^d$, the oracle returns a vector $\ora(w) \in \real^d$ satisfying
\myeqn{
    \norm{\ora(w) - \nabla f(w)} &\le& \eps \label{eqn:apg-opt:oracle}
}
where $\eps \ge 0$; we will call it an {\em $\eps$-AGP oracle}. Our goal is to understand {\em how well} and {\em fast} optimization can be done when $\ora(w)$ can be an arbitrary vector obeying \eqref{eqn:apg-opt:oracle}.

\subsection{No Norm Bounds, No Call Bounds}

We start with a result showing that if nothing is known about $\norm{w^*}$, it is impossible to guarantee a sub-optimality gap $\tau$ with a finite number of oracle calls.



\boxminipg{0.97\linewidth}{
\begin{theorem} \label{thm:agp-opt:impossible-no-bound}
    Fix any $L \ge 1$, any $\eps \ge 0$, and any $\tau > 0$. If no upper bound on $\norm{w^*}$ is known in advance, no algorithm can guarantee --- for every function $f: \real \rightarrow \real$ satisfying conditions {\bf C1}, {\bf C2}, and $\norm{\nabla f(w)} \le 3\tau$ for all $w \in \real$ --- finding a $w \in \real$ with $f(w) \le f(w^*) + \tau$ by calling an $\eps$-AGP oracle a finite number of times.
\end{theorem}
}


\begin{proof}
    It suffices to prove the theorem for $\eps = 0$ (if $\eps > 0$, simply treat a 0-AGP oracle as an $\eps$-oracle and apply the argument below). Let us define a function from $\real$ to $\real$, 
    \myeqn{
        \hspace{-5mm}
        && f(w) =
        \left\{ \begin{tabular}{ll}
                    0 &
                    \hspace{-3mm} if $w < -R$ \\
                    $(L/2) (w+R)^2$ &
                    \hspace{-3mm} if $-R \le w \le -R + 3\tau/L$ \\
                    $3\tau w + 3\tau R - (3\tau)^2 / (2L)$ &
                    \hspace{-3mm} if $w > -R + 3\tau/L$
                \end{tabular}
        \right.
        \nn
    }
    where $R > 0$ is a real value to be decided later.
    As shown in Appendix~\ref{app:proof:apg-opt:gap}, the function is convex, $L$-smooth, and ascending. Furthermore, it has the following properties:
    \myitems{
        \item Property 1: $\min_w f(w) = 0$, and $f(w) = 0$ if and only if $w \le -R$.
        \item Property 2:  $\norm{\nabla f(w)} \le 3\tau$ for all $w \in \real$, and $\nabla f(w) = 3\tau$ for all $w > -R + 3\tau / L$.
        \item Property 3: $f(w) \ge 3 \tau$ for any $w \ge -R + (3\tau/L) + 1$.
    }

    \vslit

    Let $\A$ be an algorithm that calls a 0-AGP oracle a finite number of calls. We will show that an adversary can force $\A$ to incur a sub-optimality gap larger than $\tau$. For this purpose, the adversary initially leaves open the value of $R$ and decides this value only after $\A$ has terminated.

    \vgap

    Specifically, for any value (i.e., a 1D vector) $w$ that $\A$ inquires the oracle with, the adversary always returns $3\tau$. Let $w_i$ be the value that $\A$ provides in its $i$-th call to the oracle, for $i \in [t]$ where $t$ is the (finite) number of calls. Let $w_{out}$ be the value output by $\A$. Define
    \myeqn{
        v = \min\{w_1, w_2, ..., w_t, w_{out}\}. \nn
    }
    The adversary then sets
    \myeqn{
        R = \max\{0, (3\tau/L) - v + 1\} \nn
    }
    which ensures $v \ge -R + (3\tau/L) + 1$.

    \vgap

    The reader can verify that the function $f$ --- which is now finalized --- has $\nabla f(w) = 3\tau$ at every $w \in \{w_1, w_2, ..., w_t\}$. In other words, the oracle's answers to $\A$ have been correct. Nevertheless, $f(w_{out}) \ge f(v) \ge 3 \tau$. Hence, $\A$ has failed to ensure $f(w_{out}) \le f(w^*) + \tau = \tau$.

    \vgap

    The above argument works only when the adversary does not promise any upper bound on $\norm{w^*}$ in advance. Indeed, $\norm{w^*}$ equals $R$, which can be arbitrarily big because it depends on the value of $v$.
\end{proof}

\subsection{Sub-Optimality Gap} \label{sec:apg-opt:gap}

Define
\myeqn{
    \F(R) = \set{f \mid \text{$f$ satisfies {\bf C1} and {\bf C2} with $\norm{w^*} \le R$}}.
    \label{eqn:apg-opt:F-class}
}
Motivated by Theorem~\ref{thm:agp-opt:impossible-no-bound}, we will study functions from $\F(R)$ with a {\em known} value of $R$. Even in this context, learning is still impossible if the target sub-optimality gap is too small:


\boxminipg{0.97\linewidth}{
\begin{theorem} \label{thm:apg-opt:gap}
    Fix any $R \ge 1$, any $L \ge 1$, and any $0 < \eps \le 1$. No algorithm working with an $\eps$-AGP oracle can guarantee --- for every $f \in \F(R)$ (recall that the definition of $\F(R)$ relies on $L$ due to {\bf C2}) --- returning a $w \in \real^d$ such that $f(w) < f(w^*) + \eps R / 2$. This holds true even if the algorithm calls the oracle an infinite number of times.
\end{theorem}
}


%

\begin{proof}
Let us define two functions from $\real$ to $\real$:
\myeqn{
    \hspace{-10mm} f_1(w) &=&
    \left\{ \begin{tabular}{ll}
                0 &
                \hspace{-3mm} if $w < -R$ \\
                $(L/2) (w+R)^2$ &
                \hspace{-3mm} if $-R \le w \le -R + \eps/L$ \\
                $\eps w + \eps R - \eps^2 / (2L)$ &
                \hspace{-3mm} if $w > -R + \eps/L$
            \end{tabular}
    \right.
    \hspace{3mm}
    \label{eqn:agp-opt:gap:f1} \\
    \hspace{-10mm} f_2(w) &=& f_1(-w). \label{eqn:agp-opt:gap:f2}
}
As shown in Appendix~\ref{app:proof:apg-opt:gap}:
\myitems{
    \item $f_1$ and $f_2$ are convex and $L$-smooth;
    \item $f_1$ is ascending and $f_2$ is descending;
    \item $\norm{\nabla f_1(w)} \le \eps$ and $\norm{\nabla f_2(w)} \le \eps$ for all $w \in \real$.
}
%
Furthermore, the facts below are easy to verify:
\myitems{
    \item For $f_1$, the optimal vector with the smallest norm is $w_1^* = -R$ with $f_1(w_1^*) = 0$, while for $f_2$, the optimal vector with the smallest norm is $w_2^* = R$ with $f_2(w_2^*) = 0$. Hence, both $f_1$ and $f_2$ belong to $\F(R)$.
    \item For $w \ge 0$, $f_1(w) \ge f(0) = \eps R - \eps^2 / (2L) \ge \eps R / 2$.
    \item For $w \le 0$, $f_2(w) \ge f(0) = \eps R - \eps^2 / (2L) \ge \eps R / 2$.
}

Let $\A$ be an algorithm relying on an AGP oracle to minimize $f$. We observe that $\A$ cannot always distinguish between $f = f_1$ and $f = f_2$. Indeed, regardless of $f \in \set{f_1, f_2}$, when $\A$ solicits $\nabla f(w)$ from the oracle, the oracle can always return the value \bm{0}, which obeys \eqref{eqn:apg-opt:oracle} because $\norm{\nabla f(w)} \le \eps$ for both $f = f_1$ and $f = f_2$.

\vgap

Consider an adversary that is committed to choosing $f \in \set{f_1, f_2}$ but does so only {\em after} $\A$ has terminated. Whenever $\A$ calls an oracle to inquire $\nabla f(w)$, the adversary instructs the oracle to return \bm{0}, which is correct no matter $f = f_1$ or $f_2$. Suppose that $\A$ finishes by outputting a vector $w_{out}$. The adversary sets $f = f_2$ if $w_{out} \le 0$, or $f = f_1$ otherwise. This makes sure $f(w_{out}) \ge \eps R / 2$, forcing the sub-optimality gap to be at least $\eps R / 2$.
\end{proof}

\subsection{An Optimization Algorithm} \label{sec:apg-opt:alg}

Next, we consider $f \in \F(R)$ and values of $\tau$ that are reasonably large. Consider the algorithm \ttt{AGP-opt} below.

\mytab{
    \> {\bf algorithm AGP-opt} $(K)$ \\
    \> 1. \> $w_0 \leftarrow \bm{0}$, $k \leftarrow 0$ \\ 
    \> 2.\> {\bf while} $k < K$ {\bf do}  \\
    \> 3. \>\> $g_k \leftarrow \ora(w_k)$ /* output of the $\eps$-AGP oracle */ \\
    \> 4. \>\> {\bf if $\norm{g_k} < 4\eps$} {\bf then return $w_k$} \\
    \> 5. \>\> $w_{k+1} \leftarrow w_k - \fr{1}{2L} \, g_k$ \\
    \> 6. \>\> $k \leftarrow k + 1$ \\
    \> 7. \> {\bf return $w_k$}
}


\boxminipg{0.97\linewidth}{
\begin{theorem} \label{thm:apg-opt:agpopt}
	Let $f$ be a function from $\F(R)$. For any real value $\tau \ge 5 \eps \norm{w^*}$, algorithm \ttt{\em AGP-opt} finds a $w \in \real^d$ satisfying $f(w) \le f(w^*) + \tau$ with parameter
	\myeqn{
        K = \min\{5LR^2/(4\tau), LR/(4\eps)\}.
        \label{eqn:agp-opt:K}
    }
    Furthermore, every $w_k$ in Line 3 of \ttt{\em AGP-opt} has norm at most $(17/8) R$.
\end{theorem}
}

Several remarks are in order before delving into the proof:
\myitems{
    \item Gradient descent (GD) can be regarded as a special instance of \ttt{AGP-opt} with $\eps = 0$. The early termination condition at Line 4 is needed to guarantee a sub-optimality gap sensitive to $\norm{w^*}$. The sub-optimality gap would become $O(\eps R)$ without early termination as analyzed in Appendix~\ref{app:simple-alg}.

    \item Theorem~\ref{thm:apg-opt:agpopt} extends the convergence bound of GD in a seamless manner: the first term of \eqref{eqn:agp-opt:K} captures the influence of optimization accuracy, while the second reflects the noise imposed by the $\eps$-AGP oracle. For $\eps = 0$, the second term disappears, leaving $K = 5LR^2/(4\tau)$, which asymptotically matches the convergence bound of GD \cite[Corollary 3.5]{gg23}.

    \item Setting $w_0 = \bm{0}$ at Line 1 is needed for the claim (in Theorem~\ref{thm:apg-opt:agpopt}) that $\norm{w_k} \le (17/8)R$ for all $k$. The claim will be useful in Section~\ref{sec:dl:q2}. Our analysis can be easily adapted to other choices of $w_0$ by shifting the origin of the coordinate system to $w_0$.
}

\extraspacing {\bf Proof of Theorem~\ref{thm:apg-opt:agpopt}.} Define
\myeqn{
    K' = \text{the value of $k$ when \ttt{AGP-opt} terminates.}
    \label{eqn:agp-opt:K'}
}
Note that $K'$ may be less than $K$ --- this happens when the termination is at Line 4. The following technical lemma is proved in Appendix~\ref{app:lmm:apg-opt:alg}.

\boxminipg{0.97\linewidth}{
\begin{lemma} \label{lmm:apg-opt:alg}
    Let $f$ be a function satisfying {\bf C1} and {\bf C2}. Algorithm \ttt{\em AGP-opt} outputs a vector $w \in \real^d$ such that $f(w) - f(w^*)$ is bounded by
    \myeqn{
        \eps \norm{w^*} +
        \max\Big\{4 \eps \norm{w^*}, \fr{L \cdot \norm{w^*}^2}{K}\Big\}.
        \label{eqn:agp-opt:alg:gap}
    }
    Furthermore, For any $k \in [0, K'-1]$, it holds that
    \myeqn{
        \norm{w_{k+1} - w^*}
        \le
        \norm{w_k - w^*} + \eps / (2L). \label{eqn:agp-opt:alg:gap2}
    }
\end{lemma}
}

As $\tau \ge 5 \eps \norm{w^*}$ (a condition in Theorem~\ref{thm:apg-opt:agpopt}), the value of \eqref{eqn:agp-opt:alg:gap} is at most $\tau$ when
\myeqn{
    K \ge \fr{L \norm{w^*}^2}{\tau - \eps \norm{w^*}}. \label{eqn:agp-opt:K:h1}
}
The inequality $\tau \ge 5 \eps \norm{w^*}$ implies $\tau - \eps \norm{w^*} \ge (4/5) \tau$. Hence,
\myeqn{
        \fr{L \norm{w^*}^2}{\tau - \eps \norm{w^*}}
        \le
        \fr{5 L \norm{w^*}^2}{4 \tau}. \label{eqn:agp-opt:K:h2}
    }
We can bound from above the right hand side of \eqref{eqn:agp-opt:K:h2} in two different ways:
\myitems{
    \item As $\tau \ge 5 \eps \norm{w^*}$ gives $\norm{w^*} / \tau \le 1 / (5\eps)$, we have
    \myeqn{
        \eqref{eqn:agp-opt:K:h2}
        \le
        \fr{L \norm{w^*}}{4 \eps}
        \le
        \fr{L R}{4 \eps} \nn
    }
    where the second inequality used $\norm{w^*} \le R$ (because $f \in \F(R)$)

    \item On the other hand, from $\norm{w^*} \le R$, we directly have
    \myeqn{
        \eqref{eqn:agp-opt:K:h2} \le
        \fr{5L R^2}{4 \tau}. \nn
    }
}
It now follows from \eqref{eqn:agp-opt:K:h1} and Lemma~\ref{lmm:apg-opt:alg} that the choice of $K$ in \eqref{eqn:agp-opt:K} allows \ttt{AGP-opt} to output a $w \in \real^d$ whose sub-optimalty gap is at most $\tau$.

\vgap

For each $k \in [0, K']$, Lemma~\ref{lmm:apg-opt:alg} tells us
\myeqn{
    \norm{w_k - w^*}
    &\le&
    \norm{w_0 - w^*} + k \eps / (2L) \nn \\
    &=&
    \norm{w^*} + k \eps / (2L). \label{eqn:agp-opt:K:h3}
}
For the value of $K$ in \eqref{eqn:agp-opt:K}, it holds that $k \le K' \le K \le LR/(4\eps)$. Thus, $k \eps / (2L) \le R/8$, from which \eqref{eqn:agp-opt:K:h3} yields $\norm{w_k - w^*} \le \norm{w^*} + R/8$, leading to
\myeqn{
    \norm{w_k} \le 2\norm{w^*} + R/8 \le (17/8)R. \nn
}
This completes the proof of Theorem~\ref{thm:apg-opt:agpopt}.

\subsection{Discussion} \label{sec:apg-opt:discussion}

Optimization under an $\eps$-AGP oracle has been studied in
\cite{cnw+22,izy22,wt24,wt25}. We will focus on \cite{izy22,wt24} because the other two papers assume function classes different from ours.

\vgap

As detailed in Appendix~\ref{app:izy22}, the results of \cite{izy22} could be used to show that a slightly modified version of GD finds a vector whose sub-optimality gap is
\myeqn{
    O\Big(
    \norm{w^*} \cdot \sqrt{(LU/K) + B\eps}
    \Big)
    \label{eqn:agp-opt:izy22}
}
where $U$ is an upper bound on $f(w)$ for every vector $w$ produced at any iteration of GD, and $B$ is an upper bound on $\norm{\nabla f(w)}$ over those vectors (see Appendix~\ref{app:izy22} for the precise definitions of $U$ and $B$). On the other hand, by Lemma~\ref{lmm:apg-opt:alg}, we can find a vector whose sub-optimality gap is
\myeqn{
    O\Big(
    \norm{w^*} \cdot \max\{\eps, L\norm{w^*}/K\}
    \Big).
    \label{eqn:agp-opt:ours-comp}
}
Our bound \eqref{eqn:agp-opt:ours-comp} converges faster than \eqref{eqn:agp-opt:izy22} (specifically: $1/K$ vs.\ $1/\sqrt{K}$). As $K \rightarrow \infty$, our bound converges to $O(\norm{w^*} \cdot \eps)$, while \eqref{eqn:agp-opt:izy22} to $O(\norm{w^*} \cdot \sqrt{B \eps})$. Crucially, in the context of \cite{izy22}, a reasonable $\eps$ {\em must} be less than $B$ --- otherwise, the fact $\norm{\nabla f(w)} \le B$ allows an $\eps$-AGP oracle to {\em always} return vector $\bm{0}$, thereby giving no information to the algorithm at all. In fact, effective optimization demands $\eps \ll B$, in which case our sub-optimality gap $O(\norm{w^*} \cdot \eps)$ is considerably smaller than that of \cite{izy22}.

\vgap

\cite{wt24} presented a mirror-descent algorithm that works when $w$ comes from a domain with a finite radius $R$. In that scenario, their algorithm ensures a sub-optimality gap of
\myeqn{
    O\left(\fr{\poly(R, L)}{\sqrt{K}} + R \eps \right). \label{eqn:wt24}
}
The bound converges slower than \eqref{eqn:agp-opt:ours-comp} ($K$ vs.\ $1/\sqrt{K}$). As $K \rightarrow \infty$, \eqref{eqn:wt24} converges to $O(R\eps)$ --- note that the bound remains so even if $\norm{w^*} \ll R$; in that scenario our bound \eqref{eqn:agp-opt:ours-comp}, which converges to $O(\norm{w^*} \eps)$, is significantly smaller.

\section{DLAGP} \label{sec:dl}

We will deploy our theory in Section~\ref{sec:agp-opt} to tackle the DLAGP problem in its general form (with an arbitrary number $n$ of clients). Every mention of ``query'' henceforth refers to an approximate gradient query.

\subsection{Answers for Q1} \label{sec:dl:q1}

Recall that $w^*$ is an optimal vector of $f$ (i.e., $w^*$ minimizes the function $f$ defined in \eqref{eqn:intro:f-loss-sum}) with the smallest norm. If no upper bound is known about $\norm{w^*}$, no DLAGP algorithm can guarantee a sub-optimality gap at most $\tau$ using a finite number of queries. Indeed, any such algorithm must work in the special case of $n = 1$, which, however, gives a contradiction against Theorem~\ref{thm:agp-opt:impossible-no-bound}. This is true even if the values $L, \eps$, and $\tau$ are all known to the algorithm.

\vgap

Next, we consider $\norm{w^*} \le R$.
Recall (from our statement of {\bf Q1} in Section~\ref{sec:intro}) that $\tau_{min}$ is the minimum sub-optimality gap that we must tolerate. Theorem~\ref{thm:apg-opt:gap} implies
\myeqn{
    \tau_{min} \ge \eps R / 2.
    \label{eqn:agp-opt:tau-min-lb}
}
To see why, assume the existence of an DLAGP algorithm that can always ensure a sub-optimality gap at most $\tau$ for some $\tau < \eps R / 2$. Such an algorithm must work in the special case of $n = 1$, which, however, gives a contradiction against Theorem~\ref{thm:apg-opt:gap}. This is true even if the values $\eps, \tau$, $R$, and $L$ are all known to the algorithm.

\vgap

Next, we argue that $\tau_{min} \le 5\eps R$, i.e.,
DLAGP is solvable on any $R$-instance as long as $
\tau \ge 5 \eps R$. As $\ell_i$ is $L$-smooth for each $i \in [n]$, so is the function $f$ in \eqref{eqn:intro:f-loss-sum}.
The server only needs to execute the algorithm \ttt{AGP-opt} in Section~\ref{sec:apg-opt:alg}. To execute Line 3 of \ttt{AGP-opt}, which demands $\ora(w_k)$, the server issues a query to every client $i \in [n]$, solicits the client's reply $v_i$, and then defines
\myeqn{
    \ora(w_k) = \fr{1}{n} \sum_{i=1}^n v_i. \nn
}
Accordingly,
\myeqn{
    \norm{\ora(w_k)-\nabla f(w_k)} &=&
    \fr{1}{n} \Big\lVert\sum_{i=1}^n v_i - \sum_{i=1}^n \nabla \ell_i(w_k) \Big\rVert \nn \\
    &\le&
    \fr{1}{n} \sum_{i=1}^n \norm{v_i - \nabla \ell_i(w_k)} \nn
}
which is at most $\eps$ because $\norm{v_i - \nabla \ell_i(w_k)} \le \eps$ for each $i \in [n]$; see \eqref{eqn:intro:qry}. Theorem~\ref{thm:apg-opt:agpopt} thus permits the server to ensure a sub-optimality gap at most $\tau$ by issuing $K \cdot n$ queries, where $K$ is given in \eqref{eqn:agp-opt:K}. This yields the bound claimed in \eqref{eqn:intro:q2-bound1}.

\subsection{Answer for Q2} \label{sec:dl:q2}

Next, we improve the query bound from the previous subsection when a small failure probability $\delta$ is permitted.

\extraspacing {\bf Center Estimation.} We first take de-tour to discuss a relevant problem. Denote by $P = \set{p_1, p_2, ..., p_n}$ a set of $n$ vectors in $\real^d$ such that $\norm{p_i} \le B$ for every $i \in [n]$. Let $(s_1, s_2, ..., s_m)$ be a sequence uniformly sampled from $[n]^m$. Equivalently, each $s_j$ ($j \in [m]$) is independently drawn from $[n]$ uniformly at random. Variations of the following result have appeared before (e.g., Lemma 18 of \cite{kl17}); a proof of our version can be found in Appendix~\ref{app:sec:dl:q2}.

\boxminipg{0.97\linewidth}{
\begin{lemma} \label{lmm:dl:center-est}
    For any $t \ge 0$, by setting $m=O((B/t))^2 \log (1/\delta))$, we can guarantee
    \myeqn{
        \Pr\Big[
        \Big\lVert
        \fr{1}{n} \sum_{i=1}^n p_i -
        \fr{1}{m} \sum_{j=1}^m p_{s_j}
        \Big\rVert
        >
        t
        \Big]
        \le
        \delta. \nn
    }
\end{lemma}
}

\extraspacing {\bf A Randomized Algorithm for DLAGP.} Let us return to the DLAGP problem, assuming
\myeqn{
    \tau \ge 5.01 \cdot \eps R.
    \label{eqn:dl:tau-cond}
}
Suppose that Line 3 of \ttt{AGP-opt} asks for $\ora(w_k)$. Instead of querying all clients (as done in Section~\ref{sec:dl:q1}), we draw $m$ numbers $s_1, s_2, ..., s_m$ independently from $[n]$ uniformly at random --- effectively sampling $m$ clients with replacement. For each $j \in [m]$, the server issues a query to solicit a vector $v_{s_j}$ from client $s_j$, which comes with the guarantee $\norm{v_{s_j} - \nabla \ell_{s_j}(w_k)} \le \eps$. Then, it computes
\myeqn{
    \ora(w_k) = \fr{1}{m} \sum_{j=1}^m v_{s_j}. \label{eqn:dl:prob2-ora}
}

To decide $m$, we need:

\boxminipg{0.97\linewidth}{
\begin{lemma} \label{lmm:dl:ell_i-norm}
    Suppose that we run \ttt{\em AGP-opt} by setting $K$ as in \eqref{eqn:agp-opt:K}. For any $i \in [n]$ and $k \in [0, K]$, it holds that $\norm{\nabla \ell_i(w_k)} \le B_0 + (17/8) LR$ (see \eqref{eqn:intro:B0} for $B_0$).
\end{lemma}
}

The proof can be found in Appendix~\ref{app:sec:dl:q2}.
From now, we fix $B = B_0 + (17/8) LR$. By Lemma~\ref{lmm:dl:center-est}, for $m = O((\fr{B}{t})^2 \log \fr{1}{\delta'})$ --- where $t$ and $\delta'$ will be decided later --- it holds that
\myeqn{
        \Pr\Big[
        \Big\lVert
        \fr{1}{n} \sum_{i=1}^n \nabla \ell_i(w_k) -
        \fr{1}{m} \sum_{j=1}^m \nabla \ell_{s_j}(w_k)
        \Big\rVert
        >
        t
        \Big]
        \le
        \delta'.
        \label{eqn:dl:prob2:h1}
    }
Note that $\fr{1}{n} \sum_{i=1}^n \nabla \ell_i(w_k)$ is exactly $\nabla f(w_k)$. On the other hand, $\norm{\ora(w_k)-  \fr{1}{m} \sum_{j=1}^m \nabla \ell_{s_j}(w_k)}$, which by \eqref{eqn:dl:prob2-ora} is $\fr{1}{m} \norm{\sum_{j=1}^m v_{s_j} - \sum_{j=1}^m \nabla \ell_{s_j}(w_k)}$, is no more than $\eps$ because $\norm{v_{s_j} - \nabla \ell_{s_j}(w_k)} \le \eps$ for all $j \in [m]$. Plugging these into \eqref{eqn:dl:prob2:h1} gives:
\myeqn{
        \Pr\Big[
        \Big\lVert
        \nabla f(w_k) -
        \ora(w_k)
        \Big\rVert
        >
        t + \eps
        \Big]
        \le
        \delta'.
        \label{eqn:dl:prob2:h2}
}
In other words, \eqref{eqn:dl:prob2-ora} serves the purpose of a $(t+\eps)$-AGP. To apply Theorem~\ref{thm:apg-opt:agpopt} and Lemma~\ref{lmm:dl:ell_i-norm}, we require
\myeqn{
    \tau \ge 5 (t + \eps) R \nn
}
as $\norm{w^*} \le R$. We therefore set
\myeqn{
    t = \tau/(5R) - \eps = \Omega(\tau/R)
    \label{eqn:dl:prob2:t}
}
where the last step used \eqref{eqn:dl:tau-cond}.

\vgap

It remains to fix the value of $\delta'$. Henceforth, set $K = 5LR^2 / (4 \tau)$; by Theorem~\ref{thm:apg-opt:agpopt}, \ttt{AGP-opt} performs at most $K$ iterations. By setting $\delta'$ to $\delta/K$, we make sure that the inequality $\norm{\nabla f(w_k) - \ora(w_k)} \le t+\eps$ holds for all relevant $k$ with probability at least $1-\delta$. We thus determine the value of $m$ to be:
\myeqn{
    m = O\left(
        \left(\fr{B_0 + LR}{\tau/R}\right)^2 \log \fr{K}{\delta}
    \right)
    =\tO
    \left(
        \fr{(B_0 + LR)^2 R^2}{\tau^2}
    \right)
    \label{eqn:dl:m}
}
where $\tO(.)$ hides factors logarithmic to $L, R, 1/\tau$, and $1/\delta$.

\vgap

We can now conclude that in total $m \cdot K$
queries suffice for the server to guarantee a sub-optimality gap at most $\tau$ with probability at least $1-\delta$. Plugging in the values of $K$ and $m$ yields the bound claimed in \eqref{eqn:intro:q2-bound2}.

\section{Experiments} \label{sec:exp}
\begin{figure*}
    \centering
    \resizebox{\textwidth}{!}{%
    \begin{tabular}{cccccc}
        \includegraphics[height=30mm]{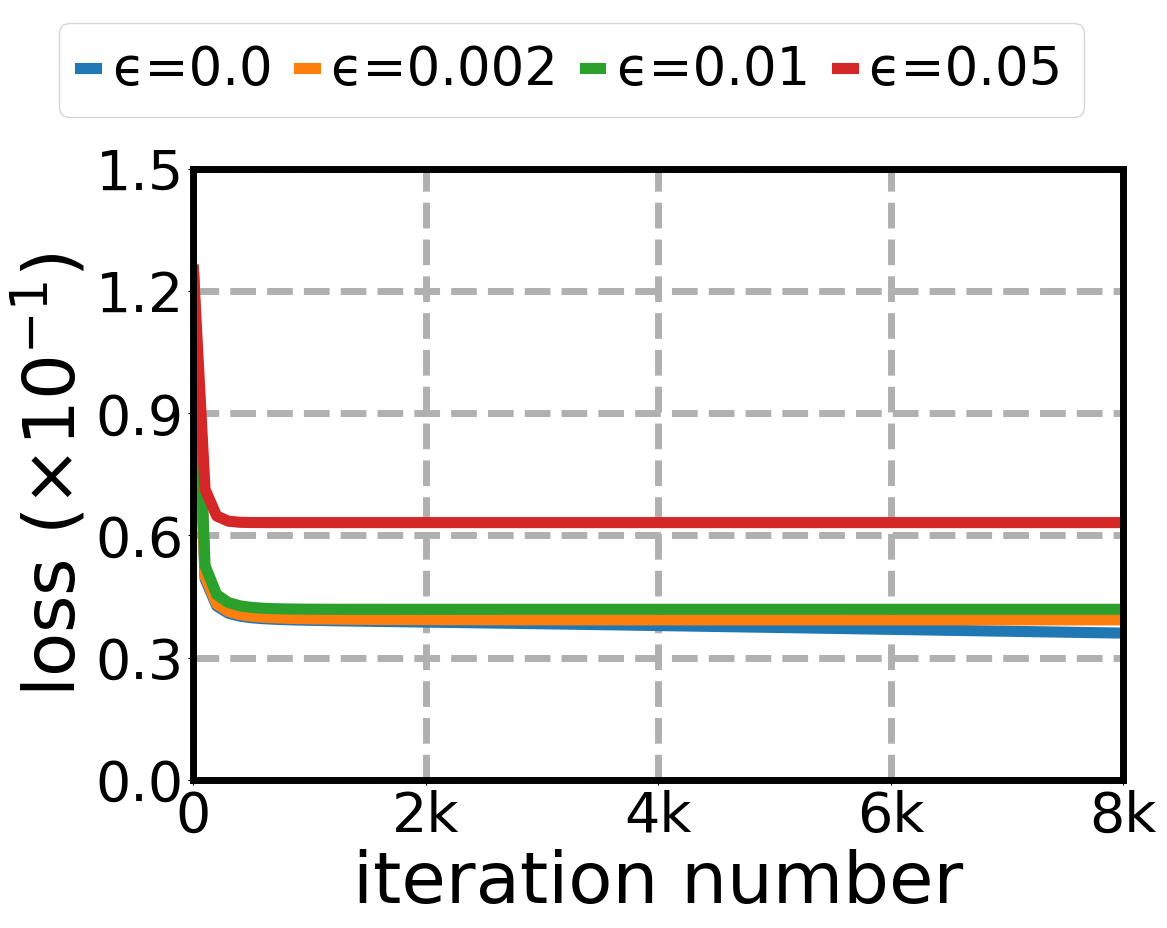} &
        \includegraphics[height=30mm]{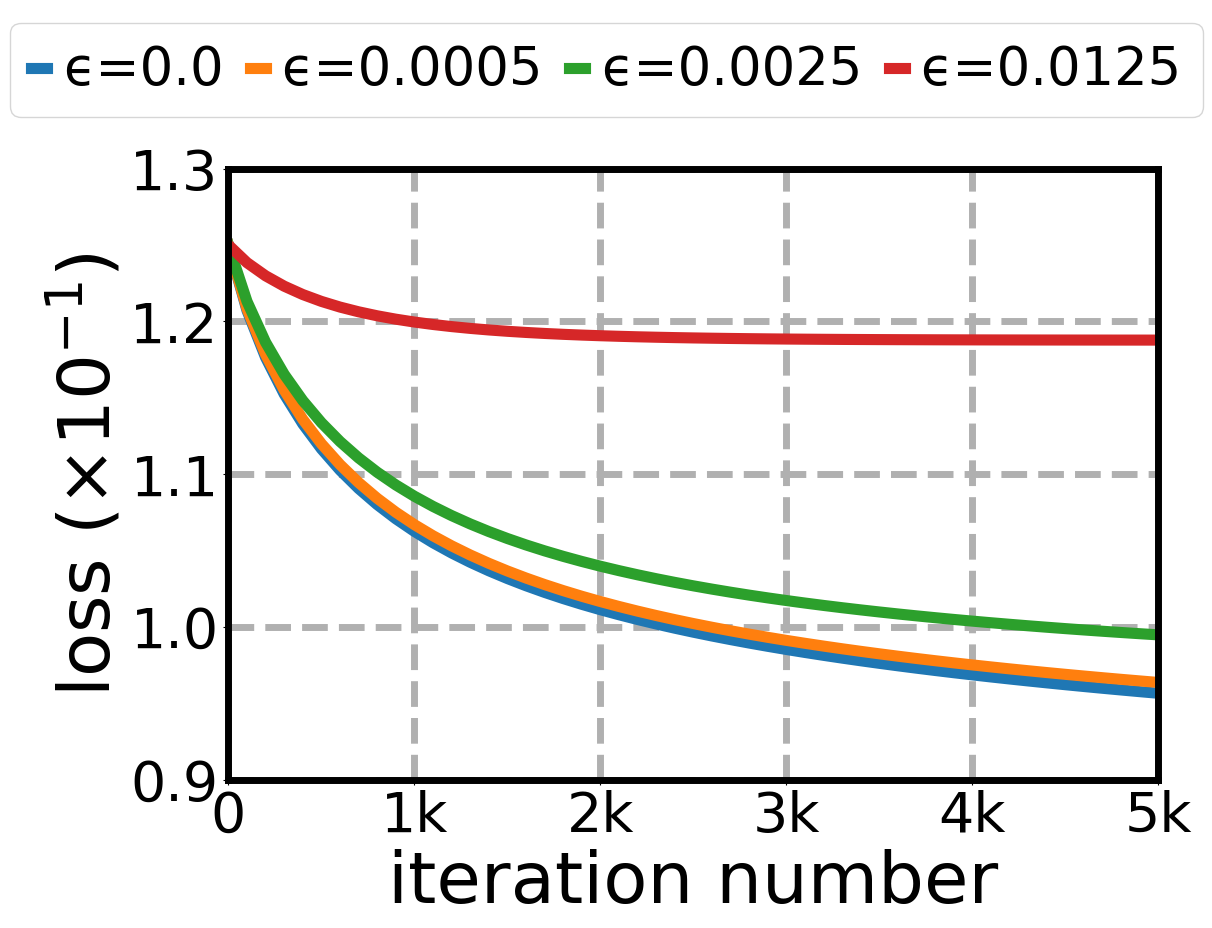} &
        \includegraphics[height=30mm]{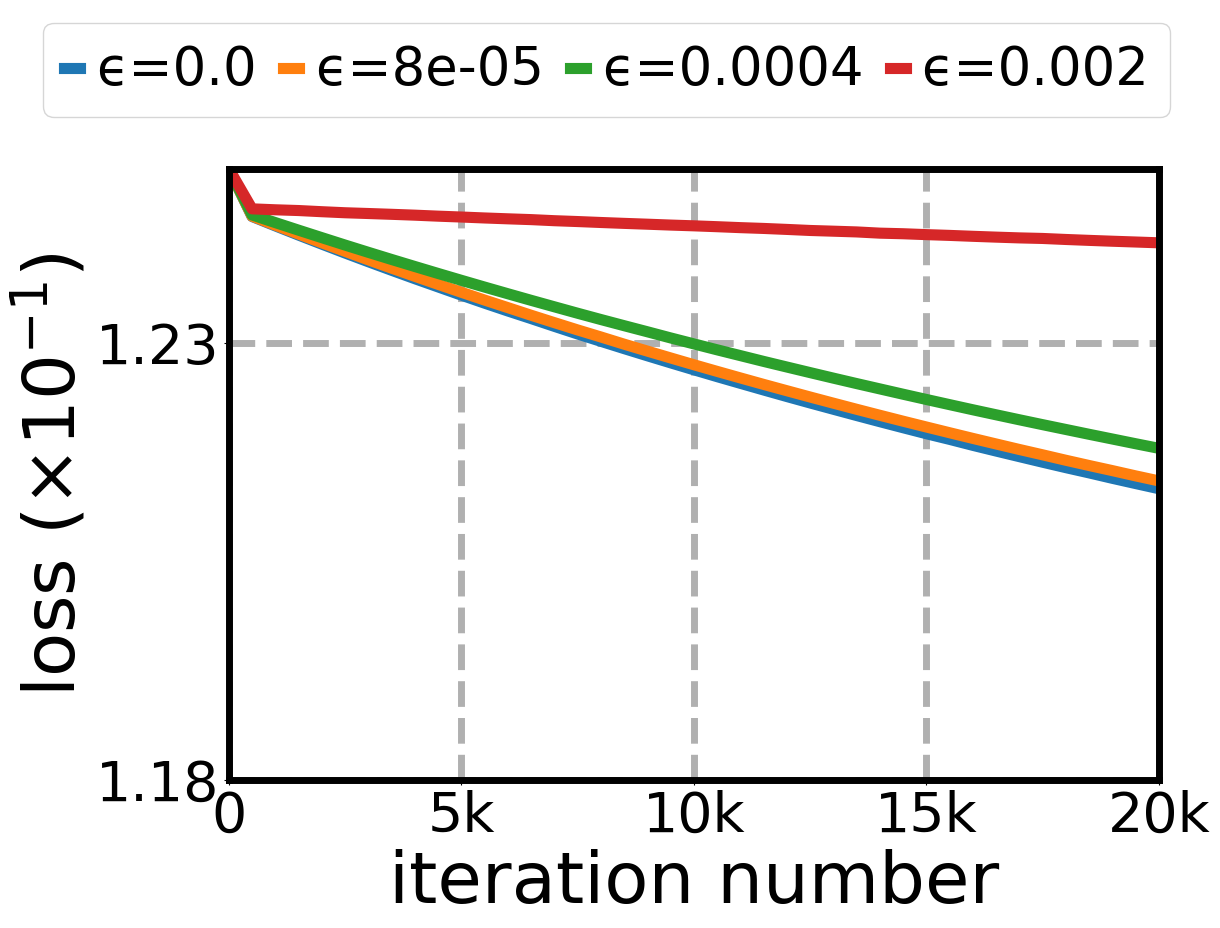} &
        \includegraphics[height=30mm]{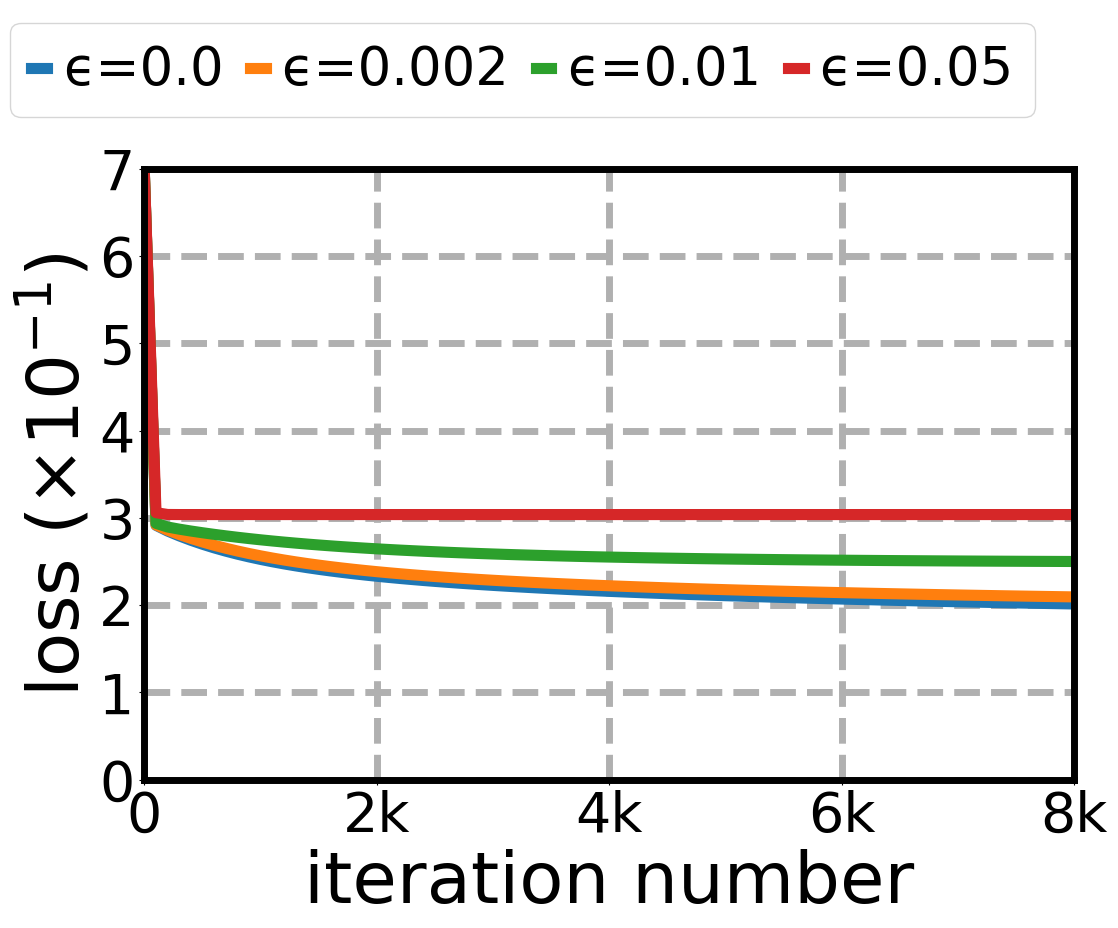} &
        \includegraphics[height=30mm]{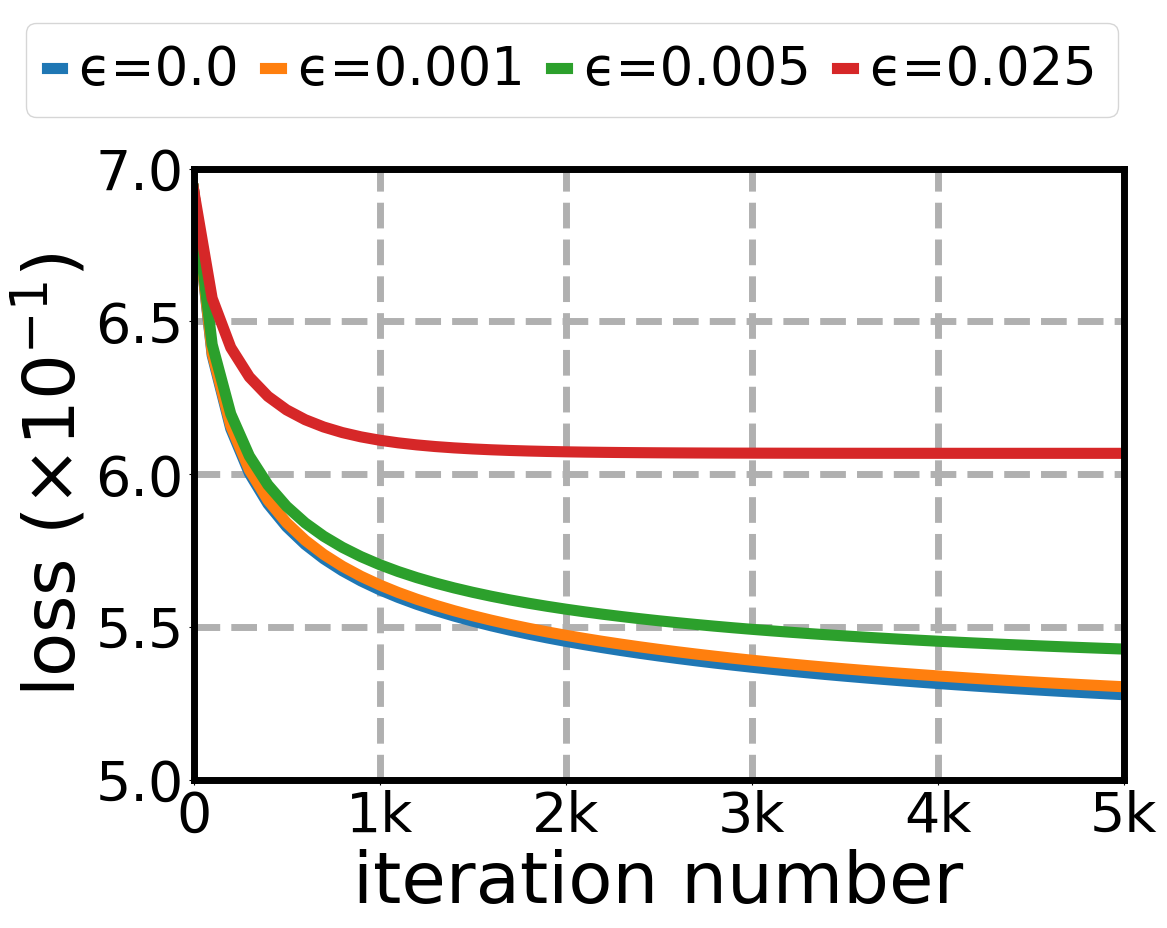} &
        \includegraphics[height=30mm]{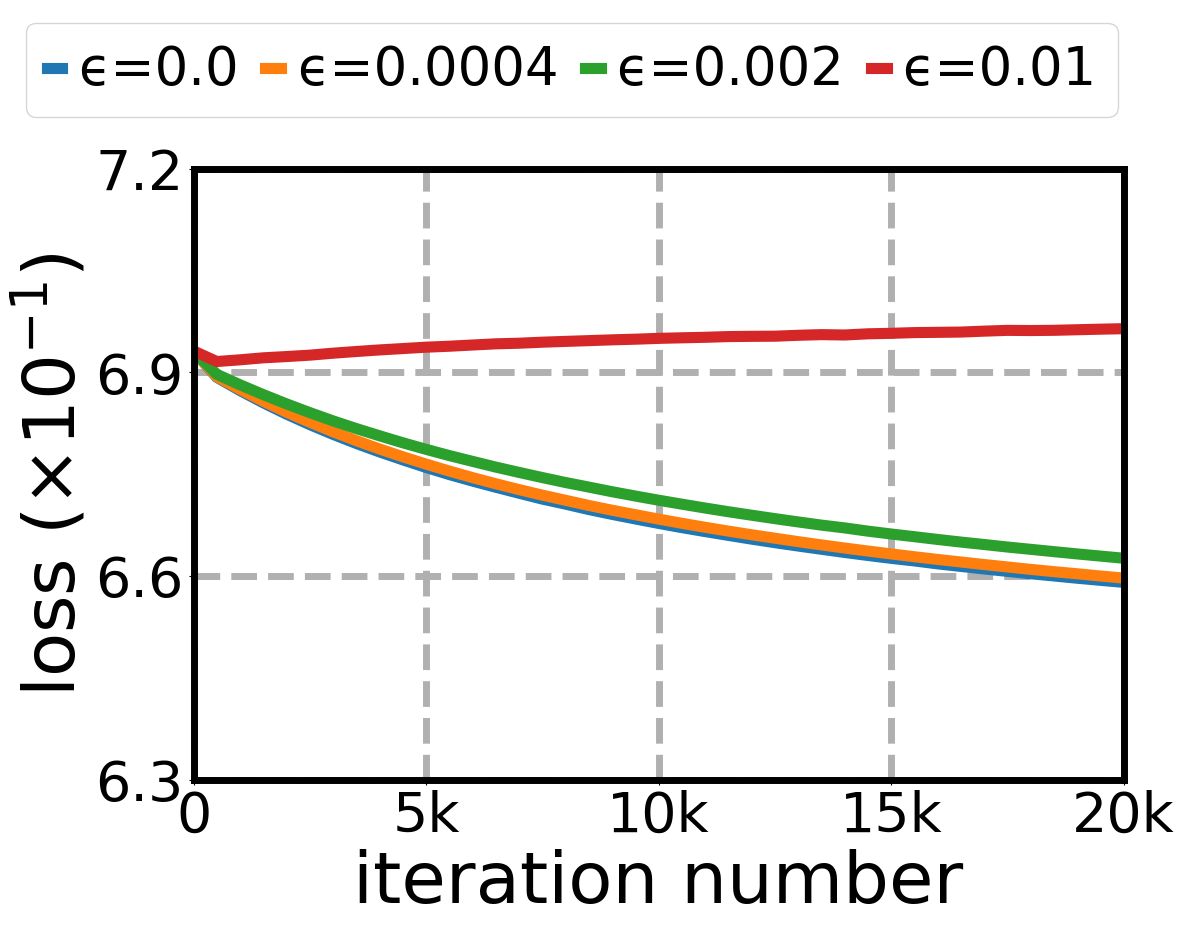}
        \\
        (a) opp-ijcnn1-rr &
        (b) opp-covtype-rr &
        (c) opp-HIGGS-rr &
        (d) opp-ijcnn1-bce &
        (e) opp-covtype-bce &
        (f) opp-HIGGS-bce \\
        \includegraphics[height=30mm]{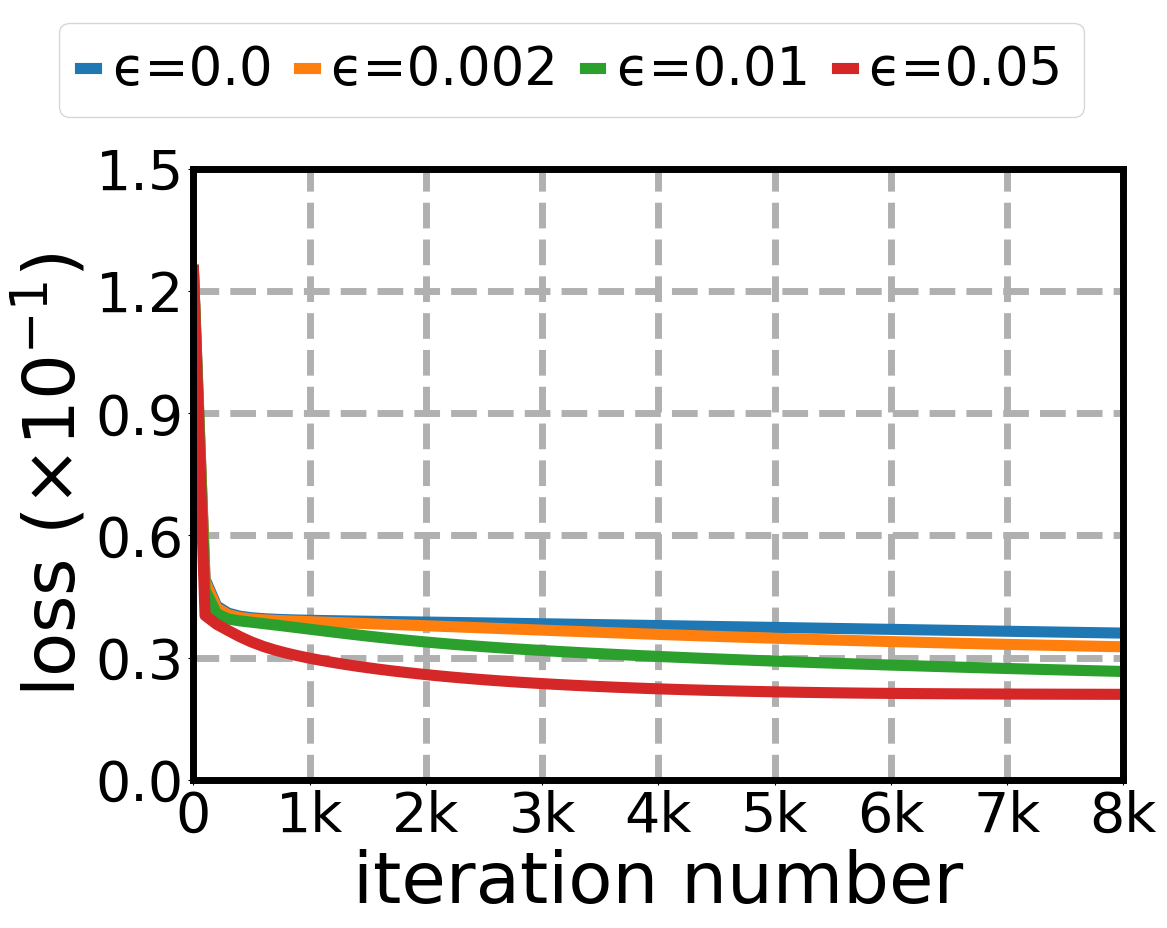} &
        \includegraphics[height=30mm]{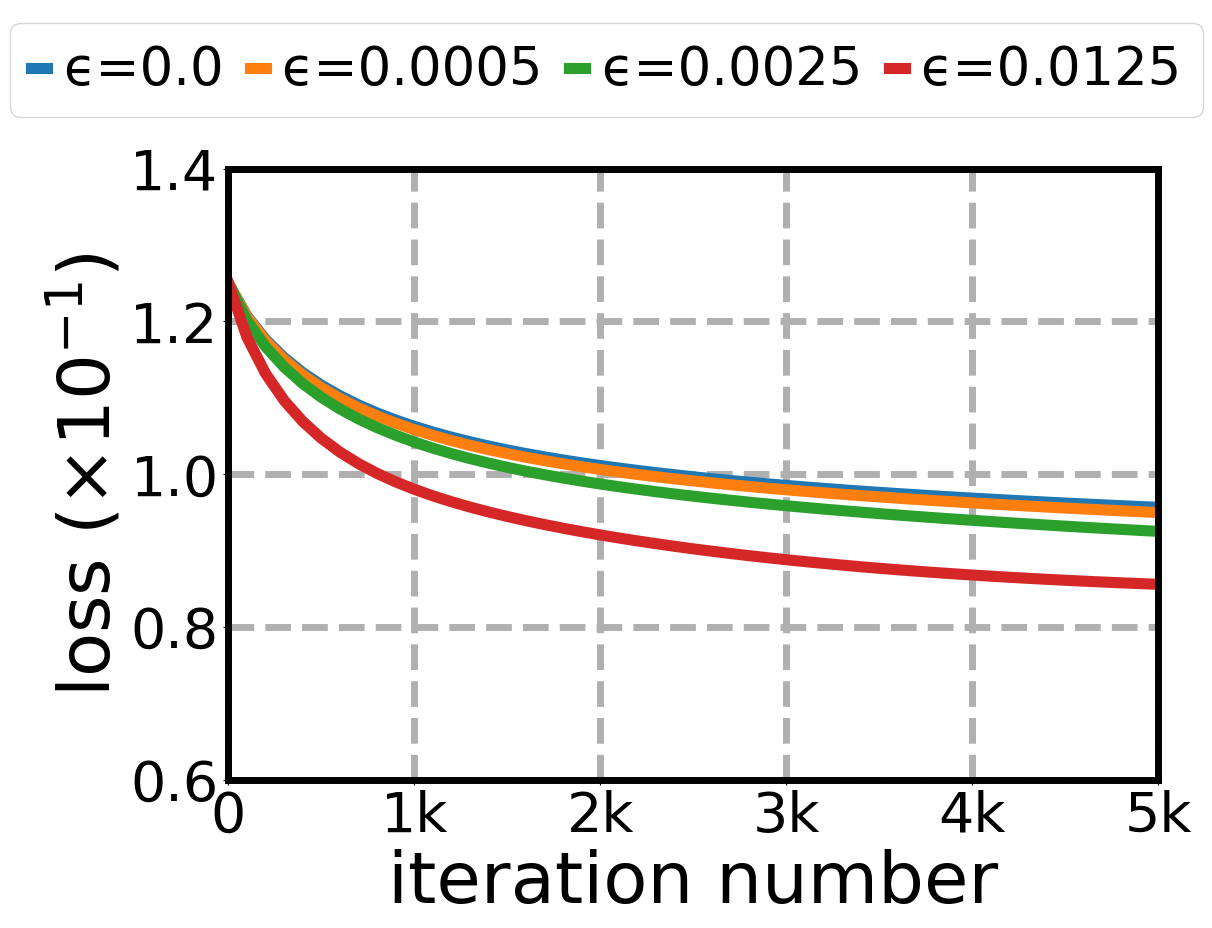} &
        \includegraphics[height=30mm]{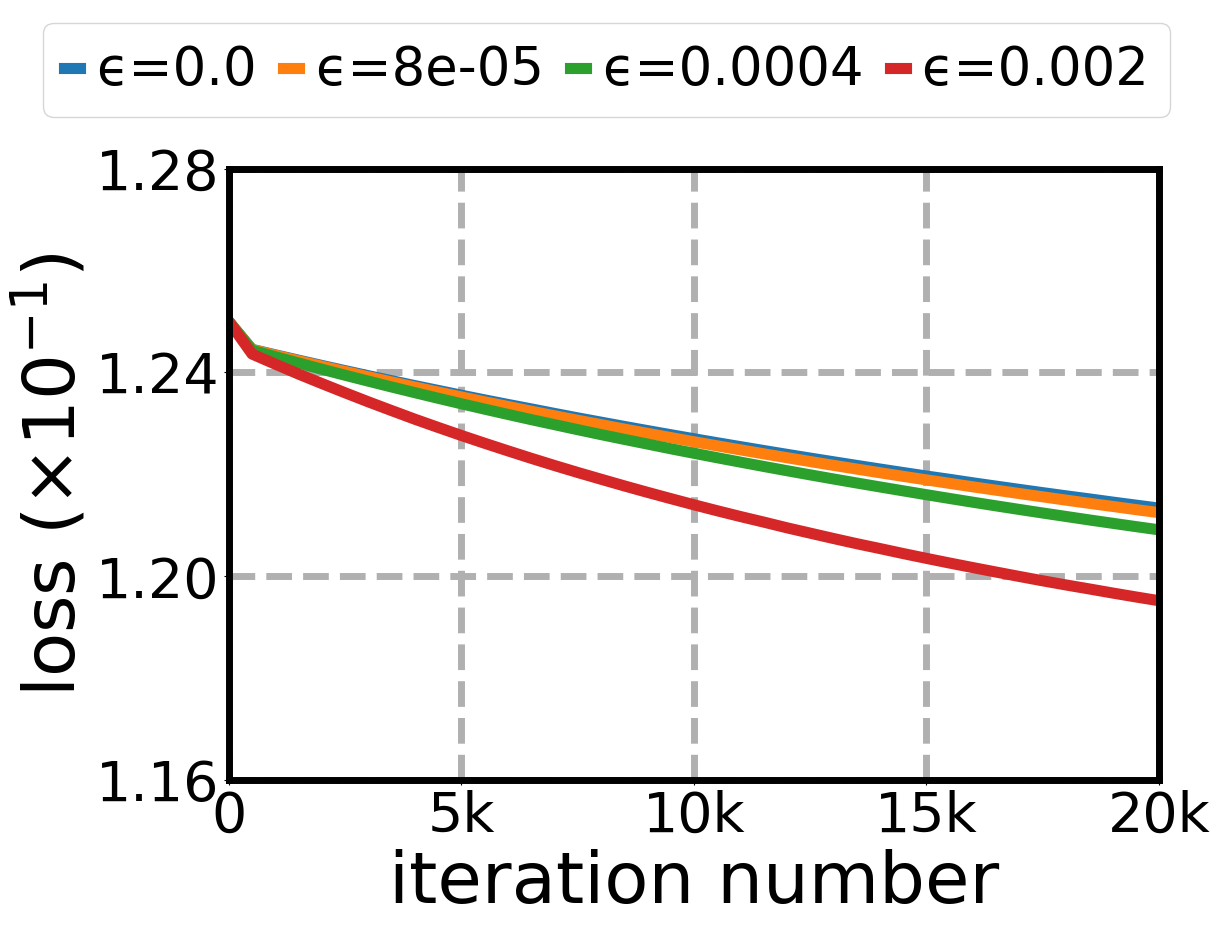} &
        \includegraphics[height=30mm]{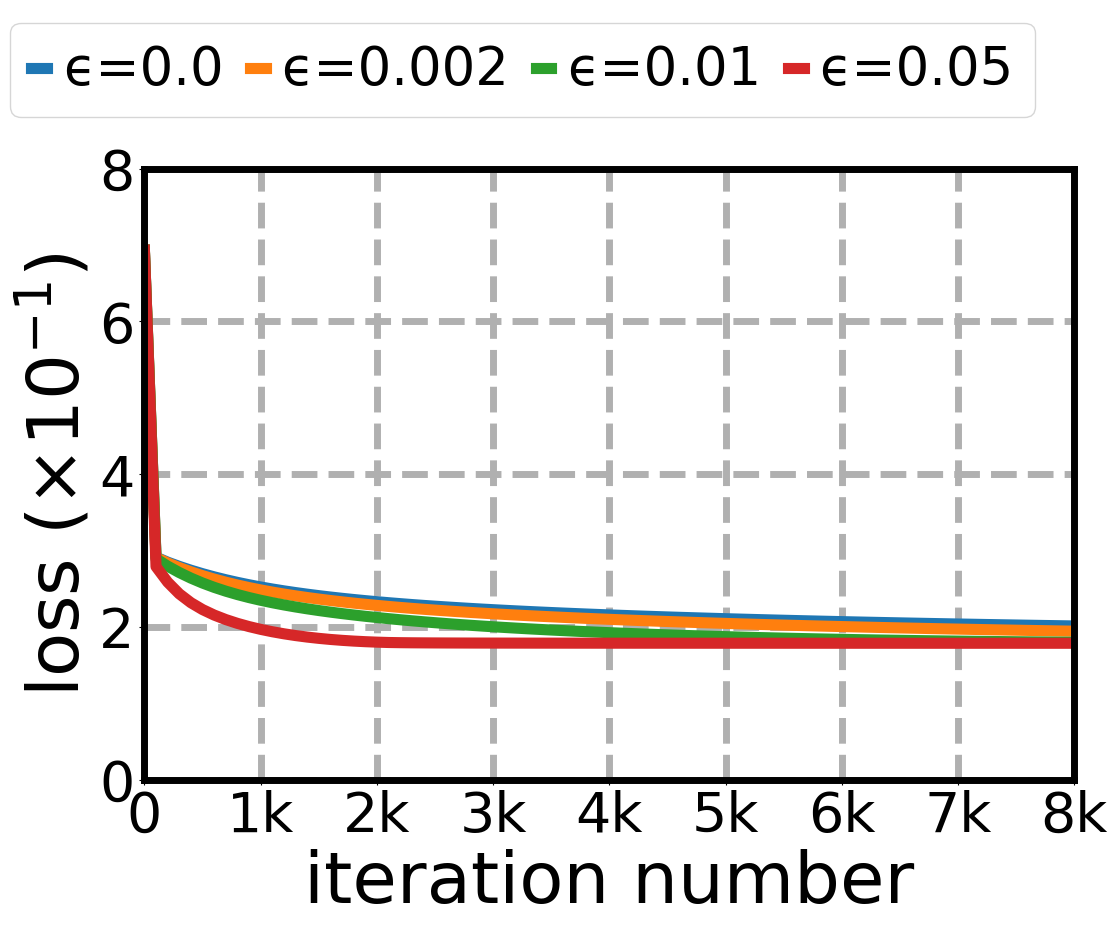} &
        \includegraphics[height=30mm]{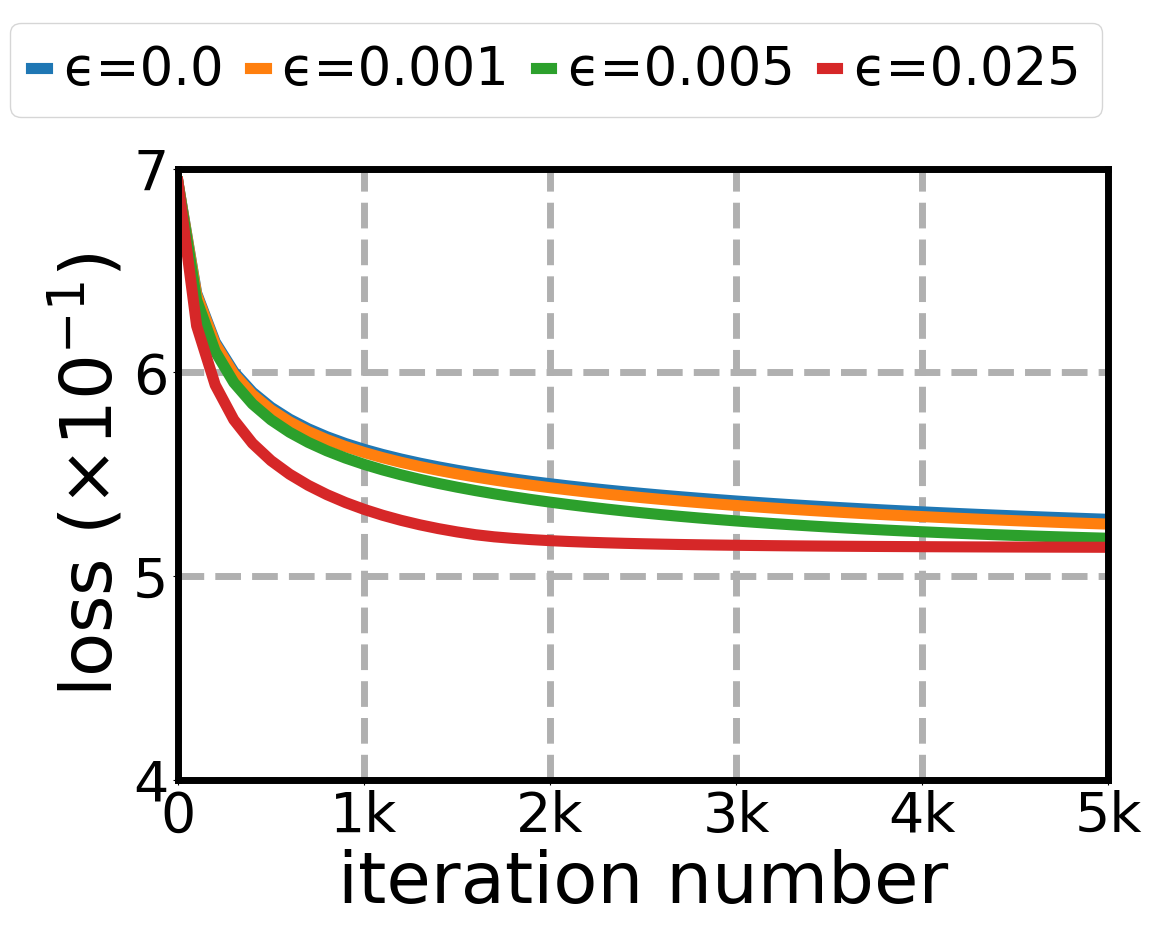} &
        \includegraphics[height=30mm]{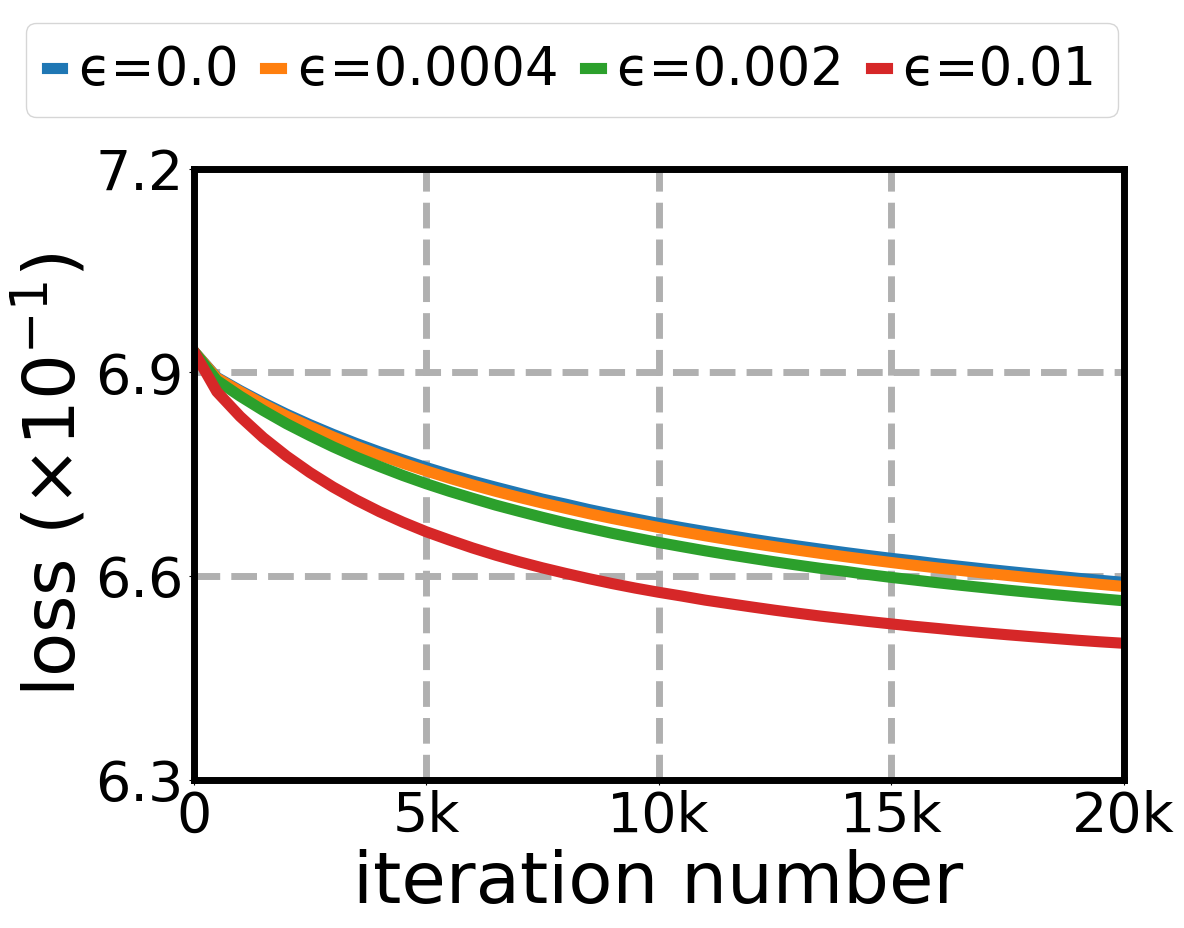}
        \\
        (g) amp-ijcnn1-rr &
        (h) amp-covtype-rr &
        (i) amp-HIGGS-rr &
        (j) amp-ijcnn1-bce &
        (k) amp-covtype-bce &
        (l) amp-HIGGS-bce \\
        \includegraphics[height=30mm]{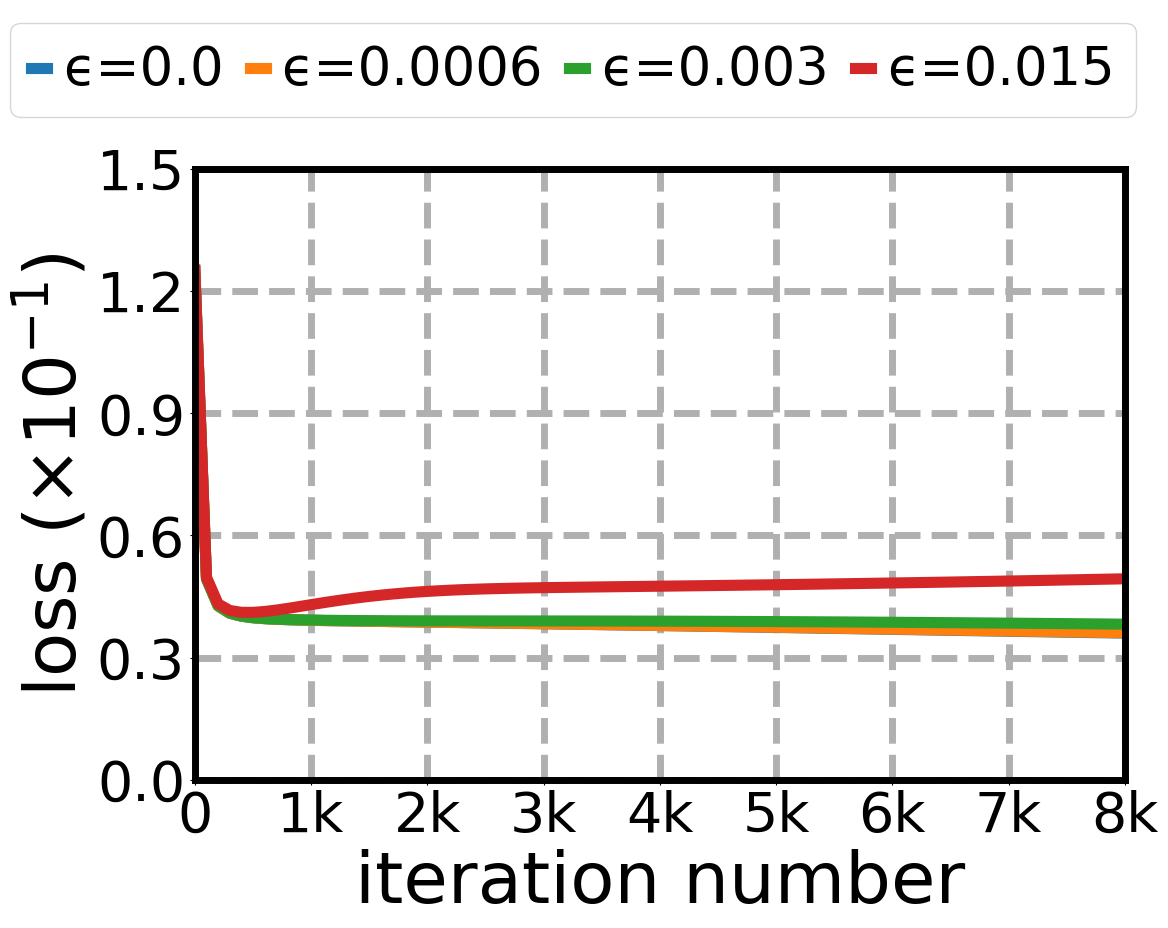} &
        \includegraphics[height=30mm]{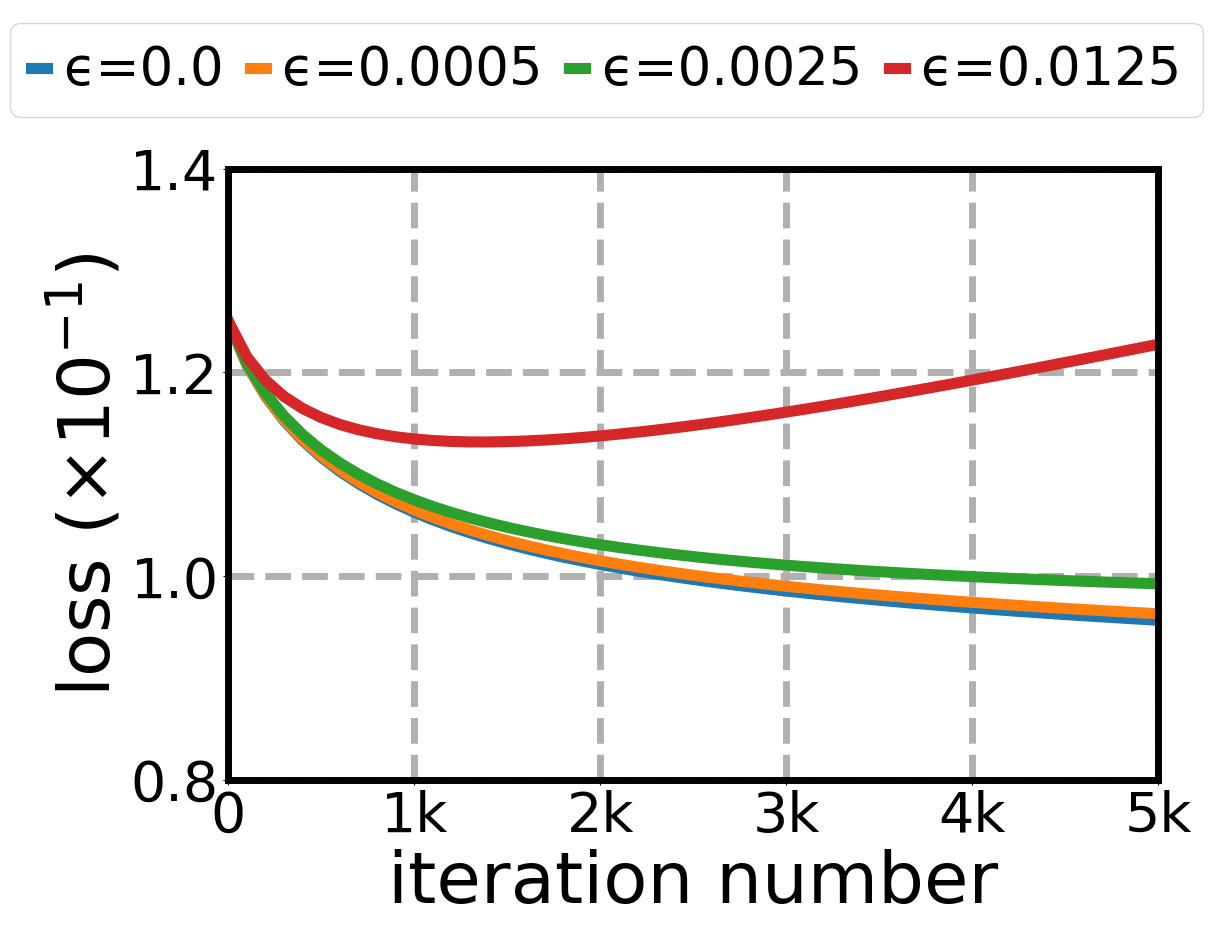} &
        \includegraphics[height=30mm]{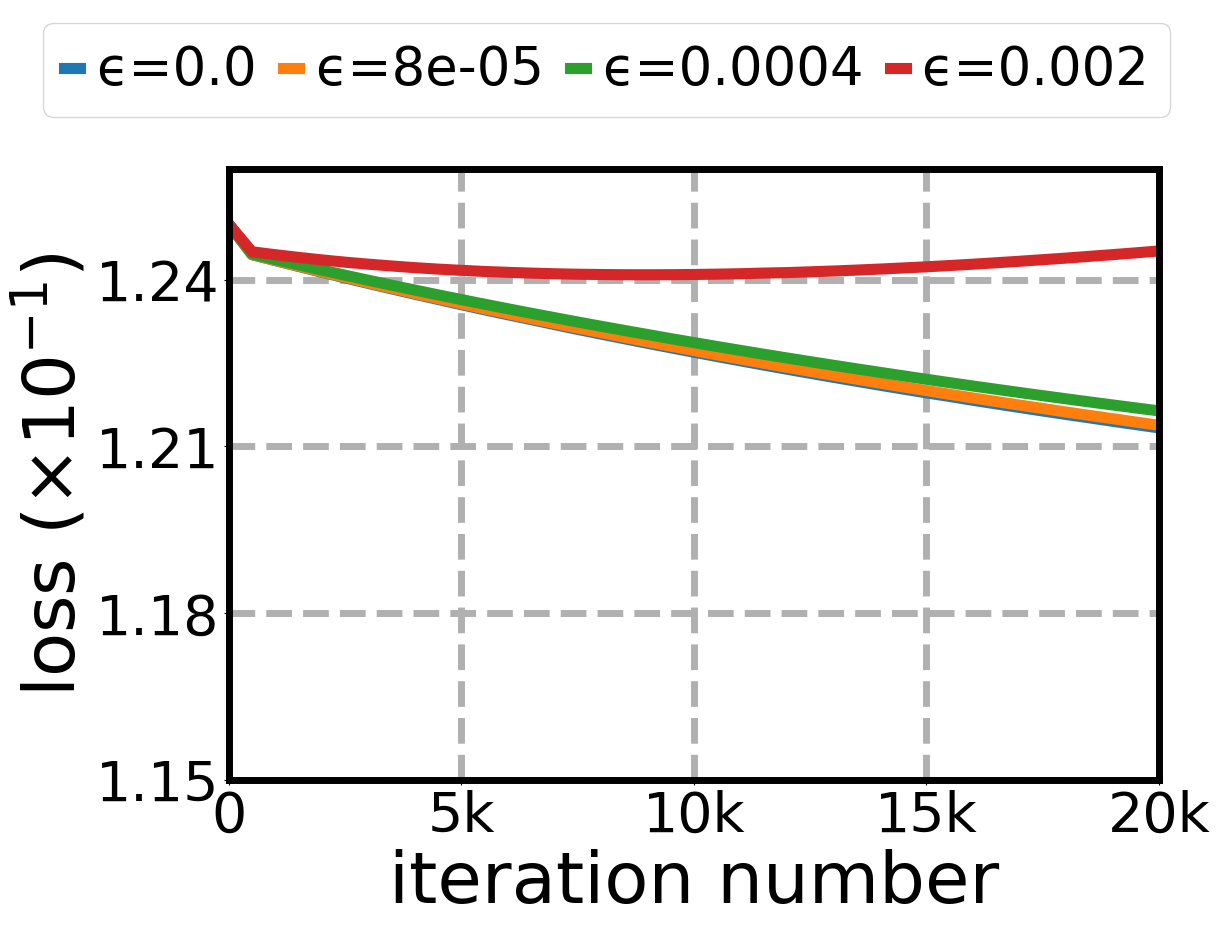} &
        \includegraphics[height=30mm]{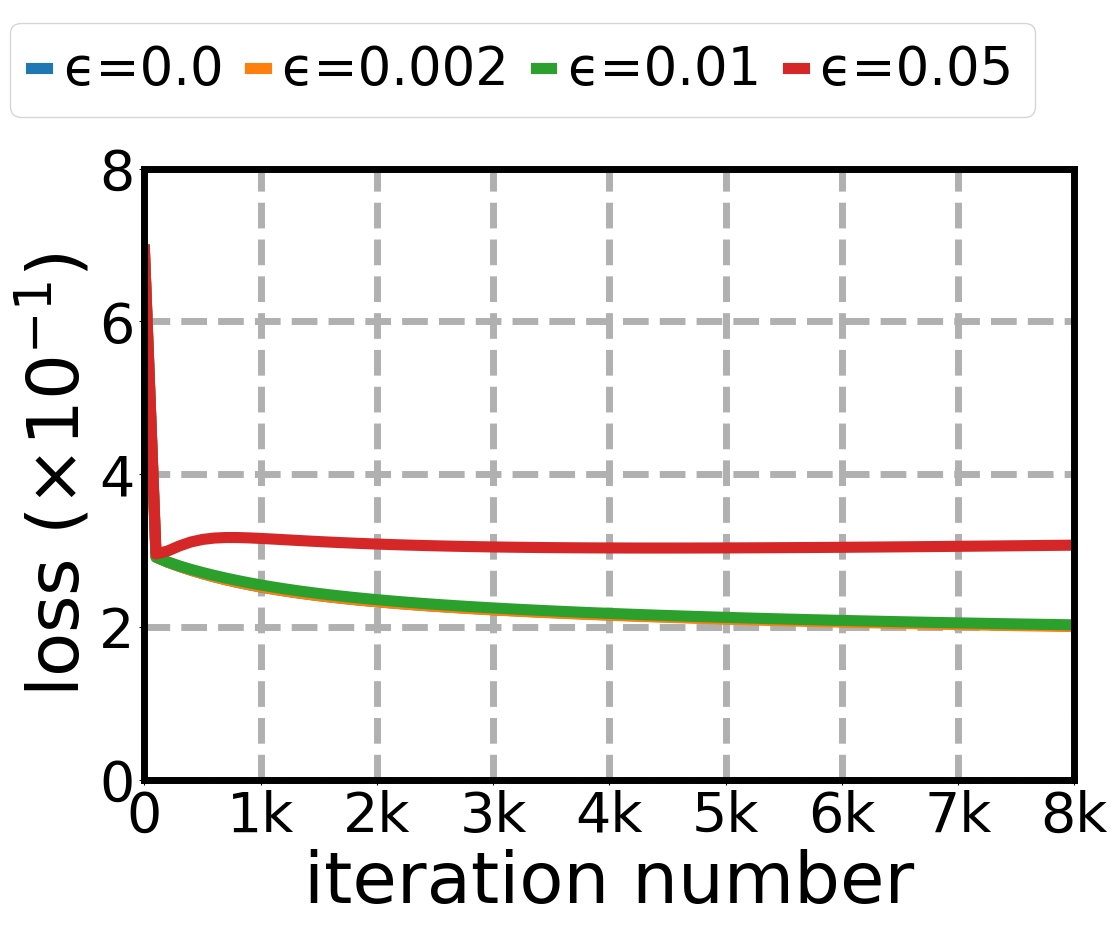} &
        \includegraphics[height=30mm]{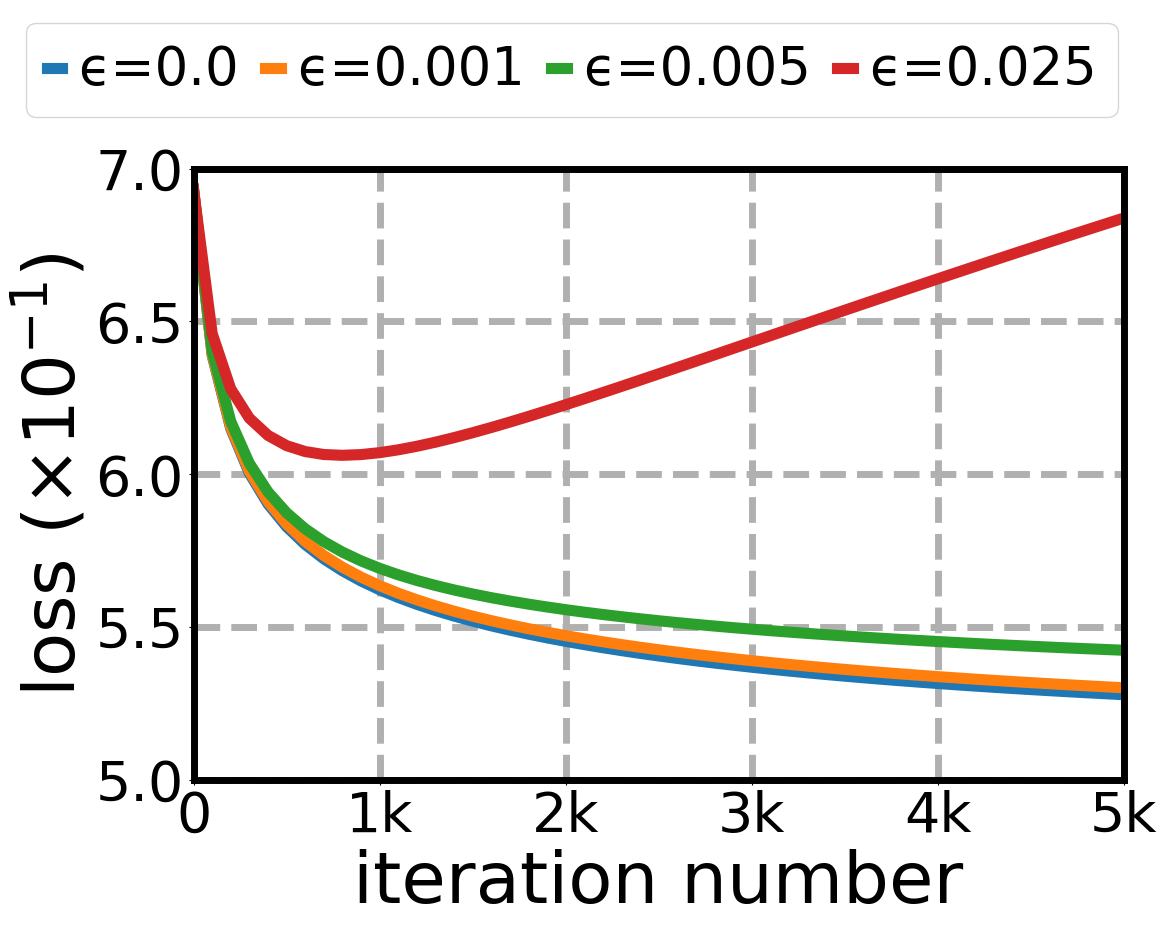} &
        \includegraphics[height=30mm]{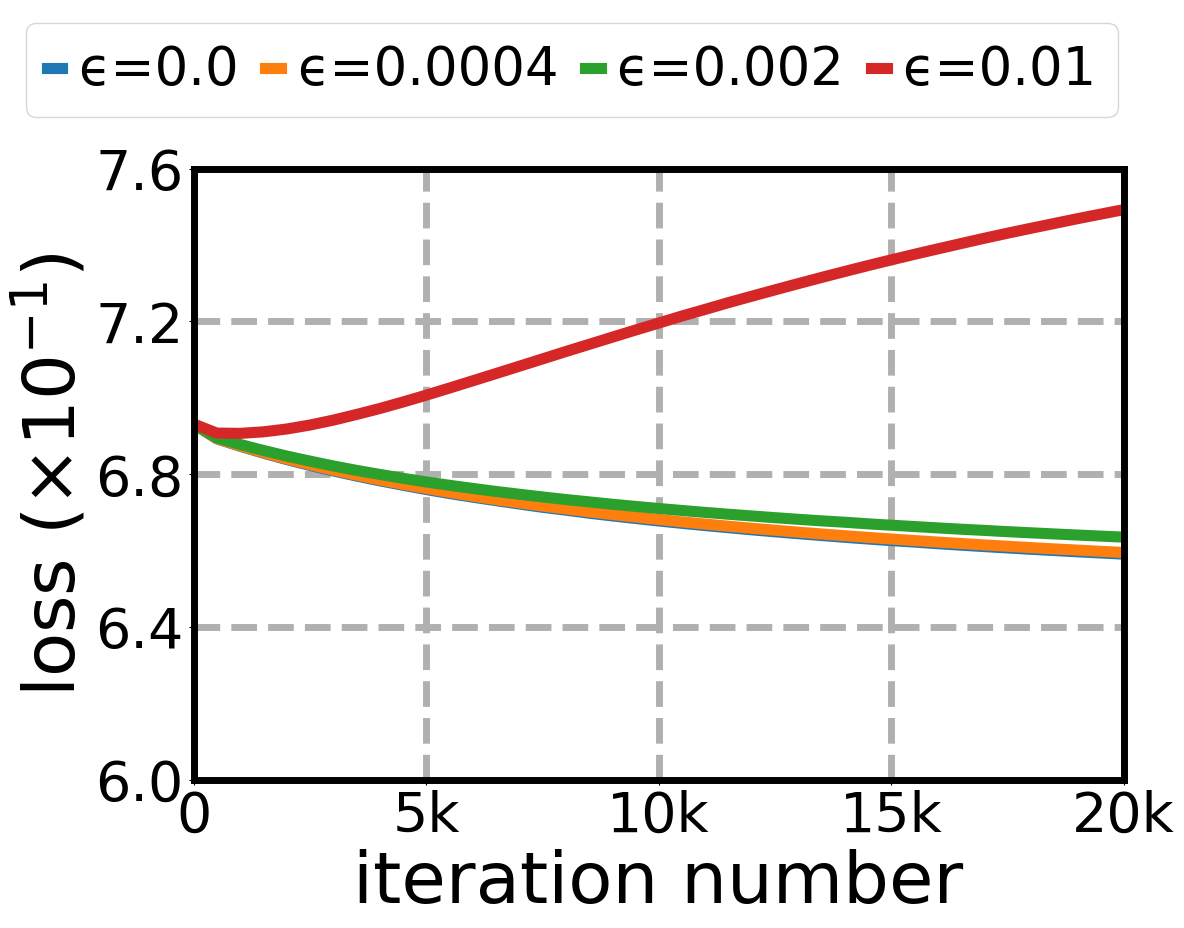}
        \\
        (m) fixed-ijcnn1-rr &
        (n) fixed-covtype-rr &
        (o) fixed-HIGGS-rr &
        (p) fixed-ijcnn1-bce &
        (q) fixed-covtype-bce &
        (r) fixed-HIGGS-bce
    \end{tabular}}
    \figcapup
    \caption{Behavior of optimization under an $\eps$-AGP oracle}
    \label{exp:agpopt}
    \figcapdown
\end{figure*}

We used three real-world datasets from the LIBSVM library\footnote{\url{Https://www.csie.ntu.edu.tw/~cjlin/libsvm}}: ijcnn1, covtype, and HIGGS; see Appendix~\ref{app:exp} for details on the normalization and cleansing steps taken. Each dataset consists of $n$ pairs $(x_i, y_i)$ ($i \in [n]$), where $x_i \in \real^d$ is a feature vector and $y_i \in \set{0, 1}$ is a binary label. The values of $(n, d)$ for ijcnn1, covtype, and HIGGS are (141691, 22 + 1), (581012, 54 + 1), and (11m, 24 + 1), respectively; see Appendix~\ref{app:exp} for the meaning of ``+1''. Each pair defines a local loss function $\ell_i(w)$ where $w \in \real^d$. The goal of optimization is to minimize the function $f(w)$ defined in \eqref{eqn:intro:f-loss-sum}. We employed two loss functions for each $i \in [n]$: (i) {\em robust regression} (RR), $\ell_i(w) = (\sigma(\dotpair{w}{x_i}) - y_i)^2$, where $\sigma : \real \rightarrow (0, 1)$ is defined as $\sigma(z) = 1/(1+e^{-z})$; and (ii) {\em binary cross-entropy} (BCE), $\ell_i(w) =$ $-y_i \log(\sigma(\dotpair{w}{x_i})) - (1-y_i) \log (1 - \sigma(\dotpair{w}{x_i}))$. Each $\ell_i$ is a convex and $L_i$-smooth function, where the value of $L_i$ was calculated based on $(x_i, y_i)$. The value $L$ (in our problem definition) was determined as $\max_{i=1}^n L_i$.

\extraspacing {\bf Optimization under AGP Oracles.} The first set of experiments studies the convergence behavior of the \ttt{AGP-opt} algorithm in Section~\ref{sec:apg-opt:alg}. We examined three $\eps$-AGP oracles:
\myitems{
    \item {{\em Opposing:}} the oracle perturbs $\nabla f(w)$ by  distance $\eps$ toward the direction {\em opposite} to that of $\nabla f(w)$;

    \item {{\em Amplifying:}} the oracle perturbs $\nabla f(w)$ by  distance $\eps$ in the same direction as $\nabla f(w)$;

    \item {{\em Fixed-Direction:}} the oracle perturbs $\nabla f(w)$ by distance $\eps$ toward a fixed direction, chosen to be the negative direction of the first dimension of $\real^d$.
}

Figure~\ref{exp:agpopt} plots the $f(w_k)$ value (labeled as ``loss'' in the diagrams) as a function of $k$ (the current iteration number) for four $\eps$ values of different magnitudes: $\eps = 0$ represents conventional gradient descent, and the other three values form a geometric progression. The caption of each diagram indicates the oracle type, dataset, and local loss functions; e.g., opp-ijcnn1-rr means that oracle is opposing, the dataset is ijcnn1, and all $\ell_1, \ell_2, ..., \ell_n$ adopt the RR function. We disabled the ``early termination'' condition at Line 4 of \ttt{AGP-opt}. In theory, the condition is needed because when $\norm{f(w)}$ drops to near $\eps$, the perturbation of an $\eps$-AGP oracle becomes significant and, in the worst case, can be made to harm optimization (as analyzed in Section~\ref{app:simple-alg}, the provable sub-optimality gap is higher without such a condition). In practice, however, such worst-case scenarios rarely happen, thus allowing further optimization even beyond the moment of early termination.

\vgap

The behavior of an opposing oracle in Figures~\ref{exp:agpopt}(a)-(f) contrasts that of an amplifying oracle in Figures~\ref{exp:agpopt}(g)-(l). In essence, an opposing oracle suppresses the ``step length'' of optimization, while an amplifying oracle strengthens it. Under an opposing oracle, higher $\eps$ values incurred larger sub-optimality gaps, but the opposite was observed under an amplifying oracle. This seemingly surprising phenomenon is in fact consistent with the understanding that a step length chosen by theoretical analysis --- $\fr{1}{2L} \ora(w_k)$ for \ttt{AGP-opt} --- is often conservatively small, such that an amplifying oracle ``happened to'' perform optimization at a better rate.

\vgap

Figures~\ref{exp:agpopt}(m)-(r) show that, under a fixed-direction oracle, the red curve --- representing a relatively large $\eps$ value --- exhibited a V-shape. To explain why, first note that at the early stages of optimization, the real gradient $\nabla f(w_k)$ has a large norm such that perturbation has little effect. As the algorithm approaches the optimal vector $w^*$, the magnitude of $\norm{\nabla f(w_k)}$ diminishes, i.e., only small updates are intended. In this case,  perturbation dictates the update direction, pulling the algorithm past $w^*$ and thus causing the loss to go up. Eventually, the optimization converges toward the $w$ where $\nabla f(w) = (\eps, 0, 0, ..., 0)^\textrm{T}$. In fact, for small $\eps$ values such V-shapes also exist but they are difficult to observe.

\extraspacing {\bf Query Budget Allocation in Distributed Learning.} The second set of experiments investigates  distributed learning under a specific query budget $Q$, i.e., the number of queries that the server can afford to issue. We consider that the server allocates the budget as follows. First, it chooses a parameter $K \le Q$, which is the {\em total} number of iterations of \ttt{AGP-opt} (without early termination) to perform. Then, in each iteration, it issues $m = Q/K$ queries to compute $\ora(w_k)$ as in \eqref{eqn:dl:prob2-ora}. There is a tradeoff: $K$ needs to be sufficiently large for \ttt{AGP-opt} to do enough iterations to approach $w^*$, but increasing $K$ reduces $m$, which elevates the uncertainty in center estimation (see Lemma~\ref{lmm:dl:center-est}). Our purpose is to provide a guideline for choosing $K$ properly.

\vgap

\begin{figure*}
    \centering
    \resizebox{\textwidth}{!}{%
    \begin{tabular}{ccc}
        \includegraphics[height=30mm]{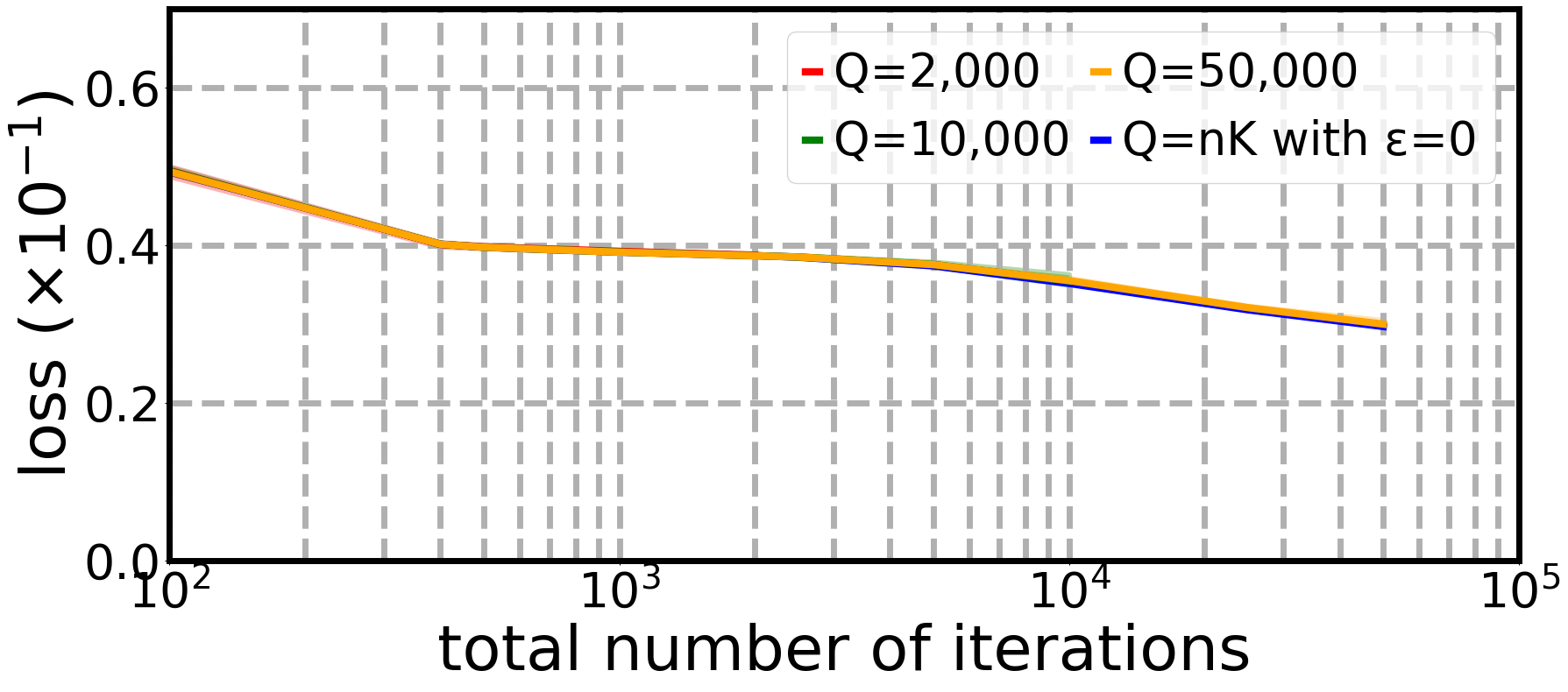} &
        \includegraphics[height=30mm]{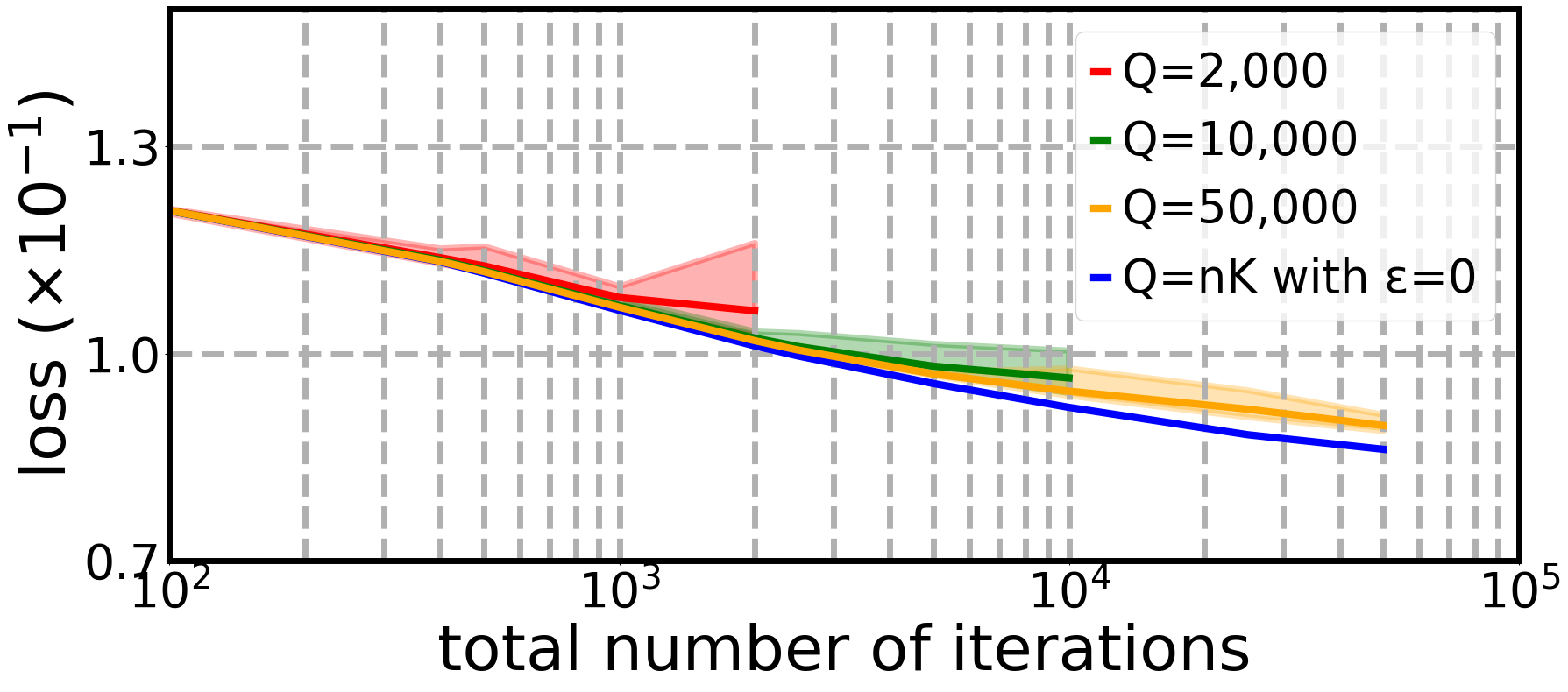} &
        \includegraphics[height=30mm]{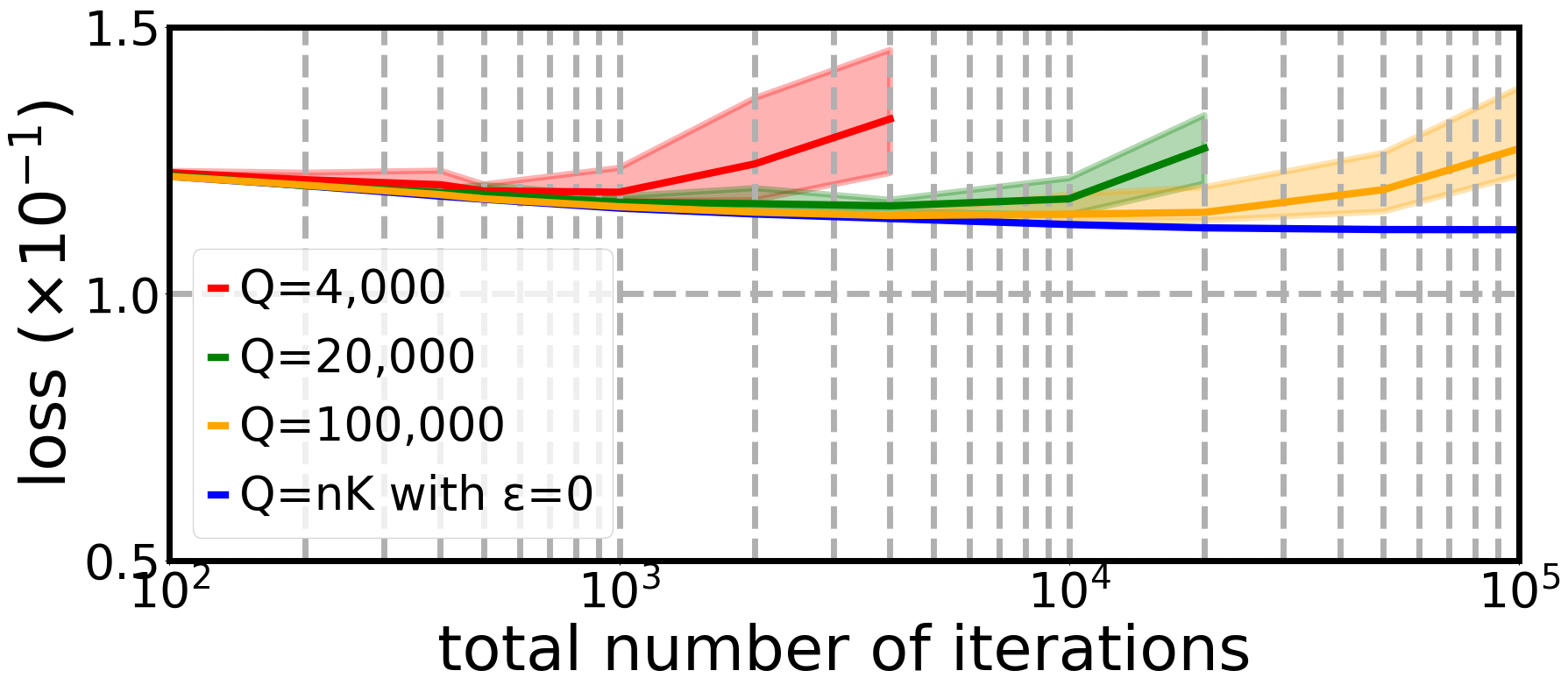} \\
        (a) ijcnn1-rr ($\eps = 0.003$) &
        (b) covtype-rr ($\eps = 0.0025$) &
        (c) HIGGS-rr  ($\eps = 0.0004$) \\
        \includegraphics[height=30mm]{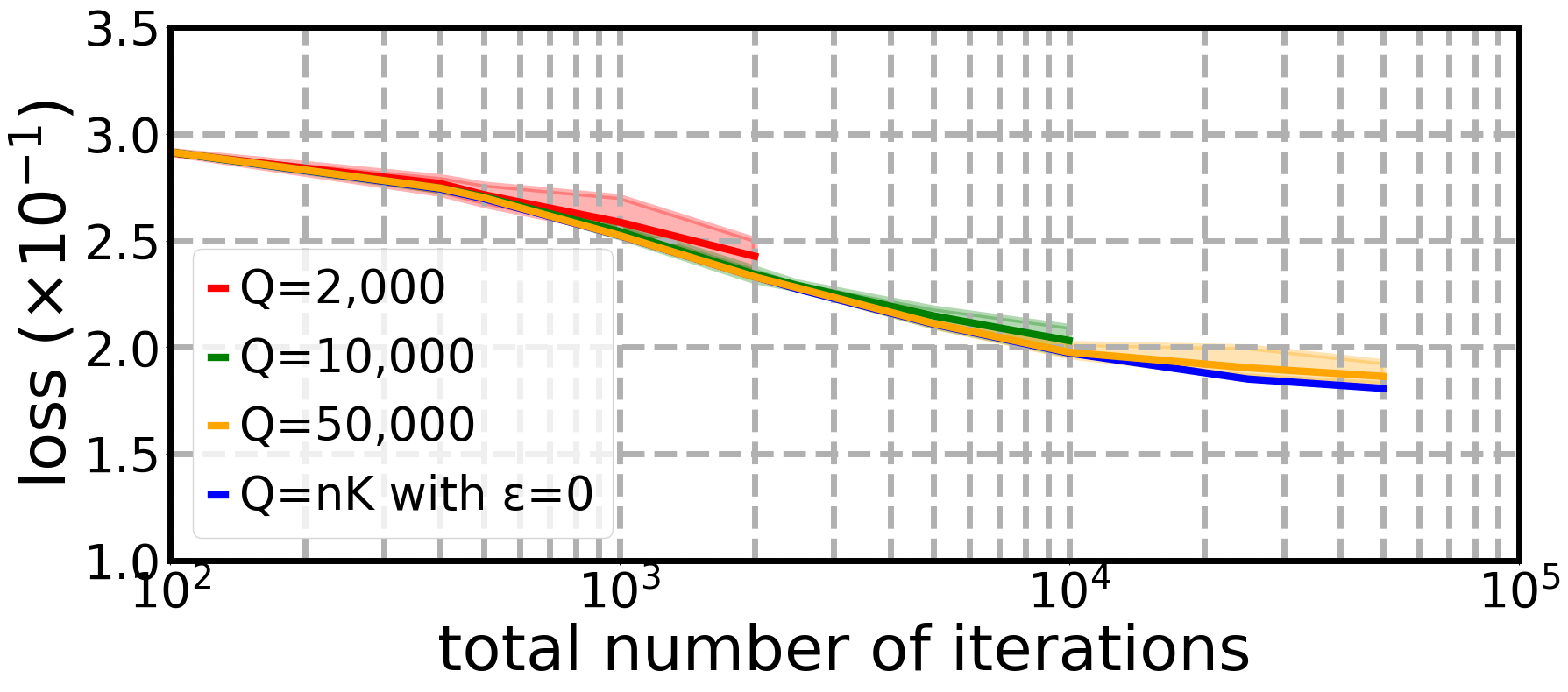} &
        \includegraphics[height=30mm]{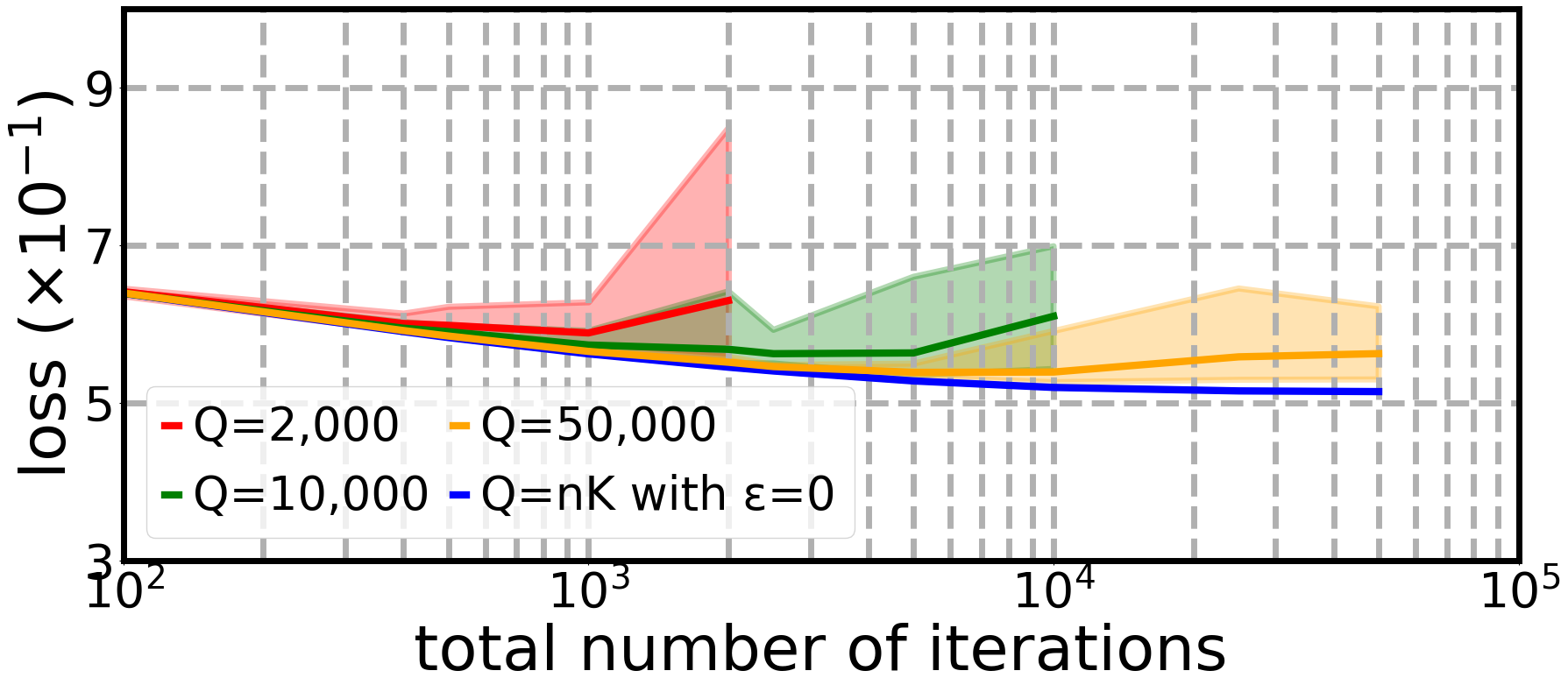} &
        \includegraphics[height=30mm]{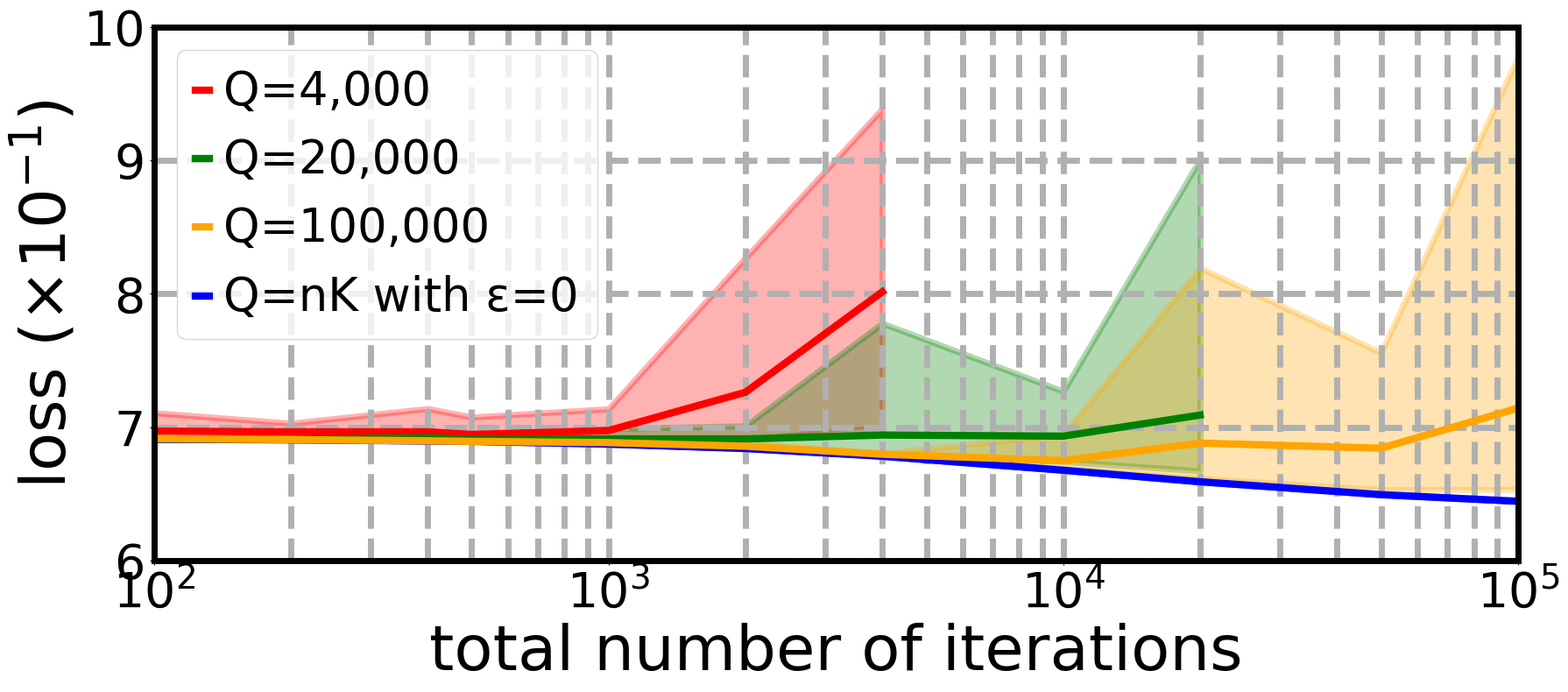}
        \\
        (d) ijcnn1-bce ($\eps = 0.01$) &
        (e) covtype-bce ($\eps = 0.005$) &
        (f) HIGGS-bce ($\eps = 0.002$)
    \end{tabular}}
    \figcapup
    \caption{Effects of query allocation in distributed learning}
    \label{exp:dl}
    \figcapdown
\end{figure*}

Figure~\ref{exp:dl} plots, as a function of $K$, the final outcome of optimization --- namely, the value $f(w_K)$ obtained by \ttt{AGP-opt} after $K$ iterations; we will refer to $f(w_K)$ as the {\em final loss}. The caption under a diagram indicates the dataset, local loss functions, and $\eps$ value deployed. In these experiments, when queried with $w$, a client $i \in [n]$ perturbs $\nabla \ell_i(w)$ using one of the three strategies described earlier --- opposing, amplifying, or fixed-direction perturbation --- with equal probability. Each diagram includes 4 curves: the orange, green, and red curves are obtained with geometrically-increasing values of $Q$, as well as a blue curve, marked as ``$Q = nK$ with $\eps = 0$'', corresponding to the ``unrealistic'' method where the server simply queries every client in each iteration and each client performs no perturbation at all. The blue curve serves as a reference for the best achievable final loss at each $K$. To account for the randomness in the strategy described in Section~\ref{sec:dl:q2}, we use shaded regions to report the range of final loss values across multiple runs.

\vgap

Evidently, $K$ must not be too small: optimization through gradient-based adjustments {\em does} require many steps. In Figures~\ref{exp:dl}(a) and \ref{exp:dl}(d), the best $K$ turned out to be $Q$, i.e., only one query was issued per iteration. However, this is not representative: in Figures~\ref{exp:dl}(b), \ref{exp:dl}(c), \ref{exp:dl}(e), and \ref{exp:dl}(f), the shaded regions start to grow considerably after $Q/K$ has dropped below a certain threshold. On the other hand, when $Q/K \ge 100$, shaded regions are hardly noticeable in all experiments. Furthermore, at the same $K$, as long as $Q/K \ge 100$, increasing $Q$ provided little help in lowering the final loss value. The above observations suggest that $m = 100$ queries per iteration are adequate for reliable center estimation, because of which $K = Q / 100$ is a robust choice in reality.


\section{Conclusions} \label{sec:conclude}

We have presented a systematic study of distributed optimization under adversarial gradient perturbations, a setting in which each client's gradient reply to a server query may arbitrarily deviate from the true gradient, subject only to a known distortion bound. Our work addresses two fundamental questions: (i) what level of optimization accuracy can be guaranteed in the presence of such adversarial replies, and (ii) how many gradient queries are sufficient to reach a given accuracy level. We establish sharp feasibility thresholds on achievable optimization accuracy, showing  a strict regime in which no algorithm can guarantee meaningful progress, as well as a complementary regime in which reliable optimization is always possible. For the attainable regime, we develop algorithms that provably reach the target accuracy using a bounded number of queries. Future directions include proving tight bounds on the query complexity and extending this framework to richer function classes.


\bibliographystyle{plain}
\bibliography{ref}

\appendix

\section*{Appendix}
\section{Completing the Proofs of Theorems~\ref{thm:agp-opt:impossible-no-bound} and \ref{thm:apg-opt:gap}} \label{app:proof:apg-opt:gap}

This section discusses the function $f$ constructed in the proof of Theorem~\ref{thm:agp-opt:impossible-no-bound} and the functions $f_1, f_2$ in the proof of Theorem~\ref{thm:apg-opt:gap}. We will focus on $f_1$ and $f_2$ because $f$ is a special case of $f_1$ with $\eps$ set to $3\tau$.

\extraspacing {\bf Convexity and Smoothness.} We will first show that $f_1$ and $f_2$, given in \eqref{eqn:agp-opt:gap:f1} and \eqref{eqn:agp-opt:gap:f2}, are convex and $L$-smooth. In fact, since $f_2$ is the reflection of $f_1$ across the y-axis and this transformation preserves convexity and smoothness, it suffices to prove the convexity and smoothness of only $f_1$.

\vgap

The derivative of $f_1$ is
\myeqn{
    f_1'(x) =
    \left\{
    \begin{tabular}{ll}
        0 & if $x < -R$ \\
        $L(x + R)$ & if $-R \le x \le -R + \eps/L$ \\
        $\eps$ & if $x > -R + \eps/L$
    \end{tabular}
    \right.
    \label{eqn:app:proof:apg-opt:gap:2}
}
Clearly, $f_1'(x)$ is non-decreasing, implying that $f_1$ is convex.

\vgap

For $\forall x,y \in \real$, next we will prove
    \myeqn{ \label{def:smooth}
        |f_1'(x) - f_1'(y)| \leq L \cdot |x - y| \label{eqn:app:proof:apg-opt:gap:1}
    }
which asserts that $f_1$ is $L$-smooth. Assume, w.l.o.g., that $x \leq y$. The subsequent discussion proceeds in four cases.

\vgap

    \textbf{Case 1:} $y < -R$ or $-R \le x \le y \le -R + \eps/L$ or $x > -R + \eps/L$. In this case, $f_1(x)$ and $f_1(y)$ are in the same ``if-branch'' of \eqref{eqn:agp-opt:gap:f1}. The function in each if-branch of \eqref{eqn:agp-opt:gap:f1} is $L$-smooth. Hence, \eqref{eqn:app:proof:apg-opt:gap:1} holds by definition of $L$-smoothness.

    \vgap

    \textbf{Case 2:} $x < -R$ and $-R \le y \le -R + \eps/L$.
    \myeqn{
        |f'(x) - f'(y)| &=& |f'(-R) - f'(y)| \nn \\
        \explain{corollary of Case 1}
        &\leq& L \cdot |(-R) - y| \nn \\
        &\leq& L \cdot |x - y|. \nn
    }

    \vslit

    \textbf{Case 3:} $-R \le x \le -R + \eps/L$ and $y >-R + \eps/L$.
    \myeqn{
        |f'(x) - f'(y)| &=& |f'(x) - f'(-R + \eps / L)| \nn \\
        \explain{corollary of Case 1}
        &\leq& L \cdot |x - (-R + \eps/ L)| \nn \\
        &\leq& L \cdot |x - y|. \nn
    }

    \vslit

    \textbf{Case 4:} $x < -R$ and $y > -R + \eps/L$.
    \myeqn{
        |f'(x) - f'(y)| &=& |f'(-R) - f'(-R + \eps / L)| \nn \\
        \explain{corollary of Case 1}
        &\leq& L \cdot |-R - (-R + \eps/ L)| \nn \\
        &\leq& L \cdot |x - y|. \nn
    }
    This completes the proof of \eqref{eqn:app:proof:apg-opt:gap:1}.

\extraspacing {\bf Monotonicity.} Function $f_1$ (resp., $f_2$) is ascending (resp., descending) because $f'_1(x) \ge 0$ (resp., $f'_2(x) \le 0$) for all $x \in \real$.

\extraspacing {\bf Gradient Norms.} $\norm{\nabla f_1(x)}$ equals $|f'_1(x)|$ where $f'_1(x)$ is given in \eqref{eqn:app:proof:apg-opt:gap:2}. It is easy to verify that $|f'_1(x)| \le \eps$ for all $x \in \real$. For any $x \in \real$, it holds that $\norm{\nabla f_2(x)} = \norm{\nabla f_1(-x)}$, which, as just proved, is at most $\eps$.

\section{Proof of Lemma~\ref{lmm:apg-opt:alg}} \label{app:lmm:apg-opt:alg}

\subsection{Proof of \eqref{eqn:agp-opt:alg:gap}}

Recall that $
    K'$ (see \eqref{eqn:agp-opt:K'} is the value of $k$ when \ttt{AGP-opt} terminates.
That is, the algorithm outputs $w_{K'}$; and our goal is to prove that $f(w_{K'}) - f(w^*)$ is bounded by \eqref{eqn:agp-opt:alg:gap}.

\vgap

Also, define
\myeqn{
    \hspace{-8mm}
    && K'' = \left\{
    \begin{tabular}{ll}
        0 & if $f(w_1) - f(w^*) < \eps \norm{w^*}$ \\
        the largest $k \in [1, K']$ s.t. $f(w_k) - f(w^*) \ge \eps \norm{w^*}$ & otherwise \\
    \end{tabular}
    \right.
    \label{eqn:apg-opt:K-prime-prime}
}

\vgap

We will prove the next two lemmas later.

\vgap

\begin{lemma} \label{lmm:apg-opt:alg:1}
    $f(w_{k + 1}) < f(w_k)$ for every $k \in [0, K'-1]$.
\end{lemma}

\vgap

\begin{lemma} \label{lmm:apt-opt:alg:3}
    If $K'' > 0$, then
    \myeqn{
        \norm{w_{K''} - w^*} \leq \norm{w_{K''-1} - w^*} \leq ... \leq \norm{w_0 - w^*} = \norm{w^*}. \nn
    }
\end{lemma}

\vgap

If $K'' < K'$, then $f(w_{K''+1}) - f(w^*) < \eps \norm{w^*}$. In this case, Lemma~\ref{lmm:apg-opt:alg:1} tells us $f(w_{K'}) \le f(w_{K''+1}) < f(w^*) + \eps \norm{w^*}$, satisfying the requirement in \eqref{eqn:agp-opt:alg:gap}. The subsequent discussion will therefore focus on the case where $K'' = K'$.

\vgap

We will distinguish two cases based on how \ttt{AGP-opt} terminates.

\myitems{
    \item {\underline{\bf {\em Termination at Line 4.}}} By the convexity of $f$, we have:
\myeqn{
    f(w_{K'}) + \dotpair{\nabla f(w_{K'})}{w^* - w_{K'}}
    &\leq&
    f(w^*) \nn
}
where $\dotpair{.}{.}$ represents dot product. We can then derive
\myeqn{
    f(w_{K'}) - f(w^*)
    &\leq&
    -\dotpair{\nabla f(w_{K'})}{w^* - w_{K'}}, \nn \\
    &\leq&
    \norm{\nabla f(w_{K'})} \cdot \norm{w^* - w_{K'}} \nn \\
    \explain{by Lemma~\ref{lmm:apt-opt:alg:3}}
    &\le&
    \norm{\nabla f(w_{K'})} \cdot \norm{w^*}. \nn
}

To analyze $\norm{\nabla f(w_{K'})}$, recall that $\norm{\nabla f(w_{K'}) - \ora(w_{K'})}$ is at most $\eps$. Further recall that $g_{K'} = \ora(w_{K'})$ and that $\norm{g_{K'}} < 4\eps$ (because Algorithm 1 terminated at Line 4). Hence:
\myeqn{
    \norm{\nabla f(w_{K'})}
    \leq
    \norm{g_{K'}} + \eps <
    5\eps. \nn
}
We thus conclude $f(w_{K'}) - f(w^*)
    \leq
    5\eps \norm{w^*}$, which is bounded by \eqref{eqn:agp-opt:alg:gap}.

    \vgap

    \item {\underline{\bf {\em Termination at Line 7.}}} For each $k \in [0, K']$, define
    \myeqn{
        E_k = L \cdot \norm{w_k - w^*}^2 + k(f(w_k) - f(w^*) - \eps \norm{w^*}). \nn
    }
    We will prove later:

    \vgap

    \begin{lemma} \label{lmm:agp-opt:alg:gap:4}
        $E_{k+1} \le E_k$ for all $k \in [0, K'-1]$.
    \end{lemma}

    \vgap

    This yields
    \myeqn{
        E_{K'}
        \leq E_{K' - 1}
        \leq ...
        \leq E_0
        = L \cdot \norm{w_0 - w^*}^2 = L \cdot \norm{w^*}^2. \nn
    }

    Plugging the definition of $E_{K'}$ into the above and applying $\norm{w_k - w^*} \ge 0$ gives:
    \myeqn{
        K'(f(w_{K'}) - f(w^*) - \eps \norm{w^*})
        &\le&
        L \cdot \norm{w^*}^2 \nn \\
        \Rightarrow
        f(w_{K'}) - f(w^*)
        &\leq&
        \frac{L \norm{w^*}^2}{K'} + \eps \norm{w^*} \nn \\
        &=&
        \frac{L \norm{w^*}^2}{K} + \eps \norm{w^*} \nn
    }
    where the last step used $K = K'$ (as \ttt{AGP-opt} finishes at Line 7). This establishes the claim in Lemma~\ref{lmm:apg-opt:alg}.
}

It remains to prove Lemmas~\ref{lmm:apg-opt:alg:1}-\ref{lmm:agp-opt:alg:gap:4}, In the rest of the section, we introduce $\gamma = 1/(2L)$. Accordingly, Line 5 of \ttt{AGP-opt} can be written as $w_{k+1} \leftarrow w_k - \gamma g_k$.

\extraspacing {\bf Proof of Lemma~\ref{lmm:apg-opt:alg:1}.} As $f$ is $L$-smooth, \cite[Lemma 2.25]{gg23} gives
    \myeqn{
        && f(w_{k+1}) - f(w_k) \nn \\
        &\leq&
        \dotpair{\nabla f(w_k)}{w_{k+1} - w_k} + \frac{L}{2}\norm{w_{k+1} - w_k}^2 \nn \\
        &=&
        \dotpair{\nabla f(w_k)}{- \gamma g_k} +  \frac{L}{2}\norm{- \gamma g_k}^2 \nn \\
        &=&
        - \gamma \dotpair{\nabla f(w_k)}{g_k} +  \frac{\gamma^2 L}{2}\norm{g_k}^2 \nn \\
        &=&
        - \gamma \dotpair{g_k + (\nabla f(w_k) - g_k)}{g_k} +  (\gamma/4) \norm{g_k}^2 \nn \\
        &=&
        - \gamma \left( \norm{g_k}^2 + \dotpair{\nabla f(w_k) - g_k}{g_k} \right) +  (\gamma/4) \norm{g_k}^2 \nn \\
        &\le & - \gamma \left( \norm{g_k}^2 - \eps\norm{g_k} \right) +  (\gamma/4) \norm{g_k}^2 \quad \text{(here we applied $\norm{\nabla f(w_k) - g_k} \le \eps$)} \nn \\
        &=&
        (-3\gamma/4) \norm{g_k}^2 + \gamma \eps \norm{g_k} \nn \\
        &=&
        \gamma \norm{g_k} (\eps - 3 \norm{g_k} / 4). \label{eqn:app:lmm:apg-opt:alg:2}
    }
    By the definition of $K'$, we must have $\norm{g_k} \ge 4\eps$ and, thus, $\eps - 3 \norm{g_k} / 4 \le \eps - 3\eps < 0$. Hence, \eqref{eqn:app:lmm:apg-opt:alg:2} is negative, meaning that $f(w_{k+1}) < f(w_k)$.


\extraspacing {\bf Proof of Lemma~\ref{lmm:apt-opt:alg:3}.} Let us first prove a technical lemma:

\begin{lemma} \label{lmm:apt-opt:alg:2}
    For every $k \in [0, K'-1]$, it holds that
    \myeqn{
        \norm{w_{k+1} - w^*}^2 - \norm{w_k - w^*}^2
        &\leq&
        \fr{f(w^*) - f(w_{k+1}) + \eps\norm{w_k - w^*}}{L}. \nn
    }
\end{lemma}

\begin{proof}
    We derive
    \myeqn{
        \hspace{-5mm}
        &&
        L \cdot \norm{w_{k+1} - w^*}^2 - L \cdot \norm{w_k - w^*}^2 \nn \\
        \hspace{-5mm}
        &=&
        L \cdot \norm{w_{k+1} - w^*}^2 - L \cdot \norm{w_k - w_{k+1} + w_{k+1} - w^*}^2 \nn \\
        \hspace{-5mm}
        &=&
        L \cdot \norm{w_{k+1} - w^*}^2 - L \cdot \norm{w_k - w_{k+1}}^2 - 2L \cdot \dotpair{w_k - w_{k+1}}{w_{k+1} - w^*} - L \cdot \norm{w_{k+1} - w^*}^2 \nn
        \\
        \hspace{-5mm}
        &=&
        - L \cdot \norm{\gamma g_k}^2 - 2L \cdot \dotpair{\gamma g_k}{w_{k+1} - w^*} \nn
        \\
        \hspace{-5mm}
        &=&
        - (\gamma/2) \cdot \norm{g_k}^2 - \dotpair{\nabla f(w_k)}{w_{k+1} - w^*} - \dotpair{g_k - \nabla f(w_k)}{w_{k+1} - w^*} \nn \\
        \hspace{-5mm}
        &\leq&
         - (\gamma/2) \cdot \norm{g_k}^2 - \dotpair{\nabla f(w_k)}{w_{k+1} - w^*} + \eps\norm{w_{k+1} - w^*} \quad\text{(here we applied $\norm{\nabla f(w_k) - g_k} \le \eps$)} \nn \\
         \hspace{-5mm}
         &\leq&
         - (\gamma/2) \cdot \norm{g_k}^2 - \dotpair{\nabla f(w_k)}{w_{k+1} - w^*} + \, \eps \cdot (\norm{w_{k + 1 } - w_k} + \norm{w_{k} - w^*}) \nn \\
         \hspace{-5mm}
        &=&
        - (\gamma/2) \cdot \norm{g_k}^2 - \dotpair{\nabla f(w_k)}{w_{k+1} - w^*} + \, \eps\gamma\norm{g_k} + \eps\norm{w_{k} - w^*}. \label{eqn:app:lmm:apg-opt:alg:3}
    }

    By the convexity of $f$, we know
    \myeqn{
        \dotpair{\nabla f(w_k)}{w^* - w_k} \leq f(w^*) - f(w_k).
        \label{eqn:app:lmm:apg-opt:alg:4}
    }
    As $f$ is $L$-smooth,
    \cite[Lemma 2.25]{gg23} gives
    \myeqn{
        -\dotpair{\nabla f(w_k)}{w_{k+1} - w_k}
        &\leq& \frac{L}{2}\norm{w_{k+1} - w_k}^2 +
        f(w_k) - f(w_{k+1}). \label{eqn:app:lmm:apg-opt:alg:5}
    }
    Adding both sides of \eqref{eqn:app:lmm:apg-opt:alg:4} and \eqref{eqn:app:lmm:apg-opt:alg:5} yields:
    \myeqn{
        -\dotpair{\nabla f(w_k)}{w_{k+1} - w^*}
        &\leq&
        \frac{L}{2}\norm{w_{k+1} - w_k}^2 +
        f(w^*) - f(w_{k+1}). \nn
    }
    By substituting the above into \eqref{eqn:app:lmm:apg-opt:alg:3}, we have:
    \myeqn{
        \hspace{-3mm}
        \eqref{eqn:app:lmm:apg-opt:alg:3}
        &\leq&
        - (\gamma/2) \cdot \norm{g_k}^2 + \frac{L}{2}\norm{w_{k+1} - w_k}^2 + f(w^*) - f(w_{k+1})
        + \, \eps\gamma\norm{g_k} + \eps\norm{w_{k} - w^*} \nn \\
        \hspace{-3mm}
        &=&
        f(w^*) - f(w_{k+1}) + \eps\norm{w_{k} - w^*} - (\gamma/2) \cdot \norm{g_k}^2
        + \, \frac{L\gamma^2}{2}\norm{g_k}^2 + \eps\gamma\norm{g_k}\nn \\
        \hspace{-3mm}
        &=&
        f(w^*) - f(w_{k+1}) + \eps\norm{w_{k} - w^*} - 
        \frac{\gamma}{4}\norm{g_k}(\norm{g_k} - 4\eps) \nn \\
        \hspace{-3mm}
        &\le&
        f(w^*) - f(w_{k+1}) + \eps\norm{w_{k} - w^*} \nn
    }
    where the last step used the fact that $\norm{g_k} \ge 4\eps$.
    The lemma now follows.
\end{proof}

We are now ready to prove Lemma~\ref{lmm:apt-opt:alg:3} with an inductive argument. First, applying Lemma~\ref{lmm:apt-opt:alg:2} with $k = 0$ gives
\myeqn{
    \norm{w_1 - w^*}^2 - \norm{w_0 - w^*}^2
    &\le&
    (1/L) \cdot
    (f(w^*) - f(w_1) + \eps \norm{w_0 - w^*}) \nn \\
    &=&
    (1/L) \cdot
    (f(w^*) - f(w_1) + \eps \norm{w^*}) \nn \\
    &\le&
    0 \nn
}
where the last step used the definition of $K''$ in \eqref{eqn:apg-opt:K-prime-prime} and the fact that $K'' > 0$. Hence, $\norm{w_1 - w^*} \le \norm{w_0 - w^*} = \norm{w^*}$.

\vgap

Inductively, for a $k \le K'' - 1$, assuming $\norm{w_k - w^*} \le \norm{w^*}$, next we will prove $\norm{w_{k+1} - w^*} \le \norm{w_k - w^*}$ (which implies $\norm{w_{k+1} - w^*} \le \norm{w^*}$). For this purpose, apply  Lemma~\ref{lmm:apt-opt:alg:2} to obtain
\myeqn{
    \norm{w_{k+1} - w^*}^2 - \norm{w_k - w^*}^2
    &\le&
    (1/L) \cdot
    (f(w^*) - f(w_{k+1}) + \eps \norm{w^k - w^*}) \nn \\
    &\le&
    (1/L) \cdot
    (f(w^*) - f(w_{k+1}) + \eps \norm{w^*}) \nn \\
    && \text{(by inductive assumption)} \nn \\
    &\le&
    0 \nn
}
where the last step used the definition of $K''$ and the fact that $k+1 \le K''$.

\extraspacing {\bf Proof of Lemma~\ref{lmm:agp-opt:alg:gap:4}.}
\myeqn{
    && E_{k+1} - E_k \nn \\
    &=&
    L \cdot \left( \norm{w_{k+1} - w^*}^2 - \norm{w_k - w^*}^2\right) + (k+1)f(w_{k+1}) - kf(w_k) - f(w^*) - \eps \norm{w^*} \nn \\
    &=&
    L \cdot \left( \norm{w_{k+1} - w^*}^2 - \norm{w_k - w^*}^2\right) + f(w_{k+1}) - f(w^*) + k(f(w_{k+1}) - f(w_k))  - \eps \norm{w^*} \nn \\
    &\le&
    f(w^*) - f(w_{k+1}) + \eps\norm{w^*} + f(w_{k+1}) - f(w^*) + k(f(w_{k+1}) - f(w_k))  - \eps \norm{w^*} \nn \\
    && \explain{by Lemma~\ref{lmm:apt-opt:alg:2}} \nn \\
    &=&
    k(f(w_{k+1}) - f(w_k)) \nn \\
    &\le&
    0. \hspace{5mm} \text{(by Lemma~\ref{lmm:apg-opt:alg:1})} \nn
}

\subsection{Proof of \eqref{eqn:agp-opt:alg:gap2}}


Once again, we will use $\gamma = 1/(2L)$ as a convenient notation.
Let us first prove a technical lemma:

\begin{lemma} \label{lmm:app:lmm:apg-opt:alg:5:1}
    For any $w \in \real^d$, it holds that
    \myeqn{
        \norm{w - \gamma \nabla f(w) - w^*}
        \le
        \norm{w - w^*}. \nn
    }
\end{lemma}

\begin{proof}
    Applying Lemma~\ref{lmm:apt-opt:alg:2} with $\eps = 0$ gives:
    \myeqn{
        \norm{w_{k+1} - w^*}^2 - \norm{w_k - w^*}^2
        &\leq&
        (f(w^*) - f(w_{k+1})) / L \nn \\
        &\leq&
        0. \nn
    }
    When $\eps = 0$, $w_{k+1} = w_k - \gamma \nabla f(w_k)$. Hence, the above yields
    \myeqn{
        \norm{w_k - \gamma \nabla f(w_k) - w^*} \le \norm{w_k - w^*}. \label{eqn:lmm:app:lmm:apg-opt:alg:5:1:proof:1}
    }

    A close look at the proof of Lemma~\ref{lmm:apt-opt:alg:2} reveals that, when $\eps = 0$, the proof holds for any $w_k \in \real^d$ and $w_{k+1} = w_k - \gamma \nabla f(w_k)$. As a result, \eqref{eqn:lmm:app:lmm:apg-opt:alg:5:1:proof:1} holds for any $w_k \in \real^d$. Setting $w_k$ to the vector $w$ in the statement of Lemma~\ref{lmm:app:lmm:apg-opt:alg:5:1} proves the claim.
\end{proof}

Recall that \ttt{AGP-opt} calculates $w_{k+1}$ as $w_k - \gamma g_k$, where $g_k = \ora(w_k)$. We use this fact to derive:
\myeqn{
    \norm{w_{k+1} - w^*}
    &=&
    \norm{w_k - \gamma g_t - w^*} \nn \\
    &=&
    \norm{w_k - \gamma (\nabla f(w_k) - \nabla f(w_k) + g_t) - w^*} \nn \\
    &=&
    \norm{w_k - \gamma \nabla f(w_k) - w^* - \gamma (g_t - \nabla f(w_k))} \nn \\
    &\le&
    \norm{w_k - \gamma \nabla f(w_k) - w^*} +
    \norm{\gamma (g_t - \nabla f(w_k))} \nn \\
    &\le&
    \norm{w_k - w^*} +
    \gamma \norm{g_t - \nabla f(w_k)} \quad \text{(by Lemma~\ref{lmm:app:lmm:apg-opt:alg:5:1})} \nn \\
    &\le&
    \norm{w_k - w^*} +
    \eps/(2L) \nn
}
where the last step used $\norm{\ora(w_k) - \nabla f(w_k)} \le \eps$.

\section{A Review on [Izzo {\em et al.}, 2022]}
\label{app:izy22}


\cite{izy22} studied the behavior of gradient descent (GD) when it runs under an $\eps$-AGP, as described in the pseudocode below:

\mytab{
    \> {\bf algorithm GD} $(K)$ \\
    \> 1. \> $w_0 \leftarrow \bm{0}$, $k \leftarrow 0$ \\ 
    \> 2.\> {\bf while} $k < K$ {\bf do}  \\
    \> 3. \>\> $g_k \leftarrow \ora(w_k)$ /* output of the $\eps$-AGP oracle */ \\
    \> 4. \>\> $w_{k+1} \leftarrow w_k - \fr{1}{2L} \, g_k$ \\
    \> 5. \>\> $k \leftarrow k + 1$ \\
    \> 6. \> {\bf return $w_k$}
}

For any $K \ge 1$, \cite[Lemma 5]{izy22} proved
\myeqn{
    \min_{k=1}^K \norm{\nabla f(w_k)}
    &=&
    O\Big(
    \sqrt{
        \fr{L \cdot U}{K} + B \eps
    }
    \Big)
    \label{eqn:izy22:1}
}
where
\myeqn{
    U &=& \max_{k=1}^K f(w_k) \nn \\
    B &=& \max_{k=1}^K \norm{\nabla f(w_k)} \nn
}

Henceforth, we will use $w^+$ to represent the vector among $w_1, ..., w_K$ that achieves the minimization on the left hand side of \eqref{eqn:izy22:1}; thus, $\nabla f(\norm{w^+})$ is bounded by the right hand side of \eqref{eqn:izy22:1}. The identification of $w^+$ requires calculating the norms of the {\em precise} gradients $\nabla f(w_1)$, $\nabla f(w_2)$, ..., $\nabla f(w_K)$. Such calculation is impossible under our perturbation model (where the only source of information about $f$ is the $\eps$-AGP oracle). Nevertheless, we proceed anyway by assuming that the vector $w^+$ can be found ``by magic''.

\vgap

\cite{izy22} did not analyze the sub-optimality gap of $w^+$. Nevertheless, we can do an analysis for them by resorting to our theory. For this purpose, let us augment their algorithm by requiring it to return the current $w_k$ as soon as $\norm{g_k}$ falls below $4\eps$ (as in Line 4 of \ttt{AGP-opt}). If the algorithm terminates this way, then our Lemma~\ref{lmm:apg-opt:alg} applies; otherwise, the sub-optimality gap of $w^+$ can be bounded as follows.

\vgap

First, the convexity of $f$ indicates
\myeqn{
    f(w^+) - f(w^*) &\le& \dotpair{\nabla f(w^+)}
                        {w^+ - w^*} \nn \\
    &\le& \norm{\nabla f(w^+)} \cdot \norm{w^+ - w^*}.
    \label{eqn:izy22:2}
}
Second, Lemma~\ref{lmm:apt-opt:alg:3} tells us $\norm{w^+ - w^*} \le \norm{w^*}$. Combining this with \eqref{eqn:izy22:1} and \eqref{eqn:izy22:2} yields
\myeqn{
    f(w^+) - f(w^*) &\le&
    \norm{w^*}
    \cdot
    O\Big(
    \sqrt{
        \fr{L \cdot U}{K} + B \eps
    }
    \Big)
    \nn
}
as is the bound used in \eqref{eqn:agp-opt:izy22}.

\section{Proofs for Section~\ref{sec:dl:q2}} \label{app:sec:dl:q2}

\noindent {\bf Proof of Lemma~\ref{lmm:dl:center-est}.} We will need:

\begin{lemma} [McDiarmid's Inequality \cite{mrt18}]
    \label{lmm:mcdiarmid}
    Let $f : \X_1 \times \X_2 \times ... \times \X_m \to \real$ be a function satisfying the following condition for each $i \in [m]$:
    \myeqn{
    \sup_{x_1, ..., x_m, y}
        |f(x_1, ..., x_i, ...,x_m) - f(x_1, ..., y, ..., x_m)|
    \le c_i \nn
    }
    where $c_1, ..., c_m$ are real values. If $X_1, ..., X_m$ are independent variables in $\X_1, ..., \X_m$, respectively, then for every $t > 0$ we have
    \myeqn{
        \Pr [
        \left| f(X_1,...,X_m)
            - \expt[f(X_1,...,X_m)] \right|
        \ge t ]
        &\le&
        2 \exp\!\left( - \frac{2 t^2}{\sum_{i=1}^m c_i^2} \right).
        \nn
    }
\end{lemma}

We are now ready to prove Lemma~\ref{lmm:dl:center-est}. It will be convenient to move the coordinate system such that $\fr{1}{n} \sum_{i=1}^n p_i = \bm{0}$. Because every point of $P$ has norm at most $B$ originally, in the new coordinate system every point has  norm at most $2B$. The subsequent discussion assumes the new coordinate system.

\vgap

Given any $(x_1, x_2, ..., x_m) \in [n]^m$, define
\myeqn{
    f(x_1, ..., x_m)
    =
    \Big\lVert
    \fr{1}{n} \sum_{i=1}^n p_i
    -
    \fr{1}{m} \sum_{j=1}^m p_{x_j}
    \Big\rVert
    =
    \Big\lVert
    \fr{1}{m} \sum_{j=1}^m p_{x_j}
    \Big\rVert.
    \nn
}
For any $j' \in [m]$, we have
    \myeqn{
        && \sup_{y \in [n]}
        |f(x_1, ..., x_m) -
        f(x_1, ..., x_{j'-1}, y, x_{j'+1}, ... , x_m)| \nn \\
        &=&
        \sup_{y \in [n]}
        \Big|
        \Big\lVert\frac{1}{m}\sum_{j=1}^{m} p_{x_j}
        \Big\rVert
        -
        \Big\lVert\frac{1}{m} \Big(\sum_{j=1}^{m} p_{x_j}\Big)
        -
        p_{x_{j'}} + p_y
        \Big\rVert
        \Big|
        \nn \\
        &\leq&
        \frac{1}{m}
        \sup_{y \in [n]}
                \norm{p_{x_{j'}} - p_y}
        \nn \\
        &\leq&
        4B/m \nn
    }
where the last step used the fact that $\norm{p_i} \le 2B$ ($i \in [n]$) in the new coordinate system.

\vgap

For each $j \in [m]$, independently draw value $s_j$ from $[n]$ uniformly at random. By Lemma~\ref{lmm:mcdiarmid}, we have:
\myeqn{
        \Pr [
        \left| f(s_1,...,s_m)
            - \expt[f(s_1,...,s_m)] \right|
        \ge t/2 ]
        &\le&
        2 \exp\!\left( - \frac{2 (t/2)^2}{\sum_{i=1}^m (4B/m)^2} \right) \nn \\
        &=&
        2 \exp\!\left( - \frac{t^2 m}{32 B^2} \right).
        \label{eqn:dl:center-est:2}
    }

By Jensen's inequality
\myeqn{
    \expt[f(s_1, ..., s_m)]
    &\le&
    \fr{1}{m}
    \sqrt{
        \expt \Big[\Big\lVert
    \sum_{j=1}^m p_{s_j}
    \Big\rVert^2 \Big]
    }.
    \label{eqn:dl:center-est:1}
}
Clearly
\myeqn{
    \Big\lVert
    \sum_{j=1}^m p_{s_j}
    \Big\rVert^2
    &=&
    \sum_{j=1}^m \norm{p_{s_j}}^2
    +
    2 \sum_{1 \le j_1 < j_2 \le m}
    \dotpair{p_{s_{j_1}}}{p_{s_{j_2}}}. \nn
}
For any $1 \le j_1 < j_2 \le m$, the random vectors $p_{s_{j_1}}$ and $p_{s_{j_2}}$ are independent and their expectations are $\bm{0}$ --- recall that $\fr{1}{n} \sum_{i=1}^n p_i = \bm{0}$ in the new coordinate system. It follows that $\expt[\dotpair{p_{s_{j_1}}}{p_{s_{j_2}}}] = 0$. Thus:
\myeqn{
    \eqref{eqn:dl:center-est:1}
    =
    \fr{1}{m}
    \sqrt{
        \expt \Big[
            \sum_{j=1}^m
            \norm{p_{s_j}}^2
        \Big]
    }
    \le
    \fr{2B}{\sqrt{m}}
    \label{eqn:dl:center-est:3}
}
where the last inequality holds because $\norm{p_i} \le 2B$ for all $i \in [n]$.

\vgap

By setting $m = O((B/t)^2 \log (1/\delta))$, we can ensure at the same time that \eqref{eqn:dl:center-est:2} is at most $\delta$, and  \eqref{eqn:dl:center-est:3} is at most $t/2$. We can now conclude that $\Pr[f(s_1, ..., s_m) > t] \le \delta$, as claimed in Lemma~\ref{lmm:dl:center-est}.

\extraspacing {\bf Proof of Lemma~\ref{lmm:dl:ell_i-norm}.}
    Theorem~\ref{thm:apg-opt:agpopt} ensures $\norm{w_k} \le (17/8)R$.
    As $\ell_i$ is $L$-smooth, we have $\norm{\nabla \ell_i(w_k) - \nabla \ell_i(\bm{0})} \le L \cdot \norm{w_k} \le (17/8) LR$. The claim now follows from $\nabla \ell_i(\bm{0}) \le B_0$.

\section{Analysis of Simple Gradient Descent} \label{app:simple-alg}

This section re-examines the gradient descent (GD) algorithm whose pseudocode is presented in Section~\ref{app:izy22}. We will extend our analysis on \ttt{AGP-OPT} to show that, given the same $K$ in Theorem~\ref{thm:apg-opt:agpopt}, GD can guarantee a sub-optimality gap of $O(\eps R)$ --- note that this is not sensitive to $\norm{w^*}$.

\vgap

Define $K'$ to be the smallest value of $k \in [0, K-1]$ satisfying $g_k < 4\eps$; if no such $k$ exists, set $K' = K$. Note that this is in fact the same $K'$ as defined in \eqref{eqn:agp-opt:K'}, but rephrased in the context here. If $K' = K$ (i.e., $\norm{g_k} \geq 4 \eps$ for all $k \in [0, K-1]$), GD behaves the same way as \ttt{AGP-opt} and, hence, ensures a sub-optimality gap at most $5 \eps \norm{w^*}$.

\vgap

The rest of the section will focus on $K' < K$. As $w_{K'}$ is vector returned by \ttt{AGP-opt}, Theorem~\ref{thm:apg-opt:agpopt} tells us
\myeqn{
    f(w_{K'}) - f(w^*) &\leq& 5 \eps \norm{w^*}. \label{eqn:simple-alg:h4}
}
Define
\myeqn{
    S_\geq &=& \{ k \in [K', K-1] : \norm{g_k} \geq 4 \eps \} \nn \\
    S_< &=& \{ k \in [K', K-1] : \norm{g_k} < 4 \eps \}. \nn
}
Thus
\myeqn{
    f(w_K) - f(w_{K'})
    &=& \sum_{k = K'}^{K - 1} f(w_{k + 1}) - f(w_k) \nn \\
    &=& \sum_{k \in S_\geq} f(w_{k + 1}) - f(w_k) 
    + \sum_{k \in S_<} f(w_{k + 1}) - f(w_k). \label{eqn:simple-alg:h1}
}
The proof of Lemma~\ref{lmm:apg-opt:alg:1} essentially shows that, for any $k \in [0,K-1]$ satisfying $\norm{g_k} \geq 4 \eps$, we have $f(w_{k + 1}) \le f(w_k)$, implying
\myeqn{
    \sum_{k \in S_\geq} (f(w_{k + 1}) - f(w_k)) &\leq& 0. \label{eqn:simple-alg:h2}
}
For any $k \in S_<$, by the convexity of $f$, we have
\myeqn{
    f(w_{k + 1}) - f(w_k)
    &\leq& \dotpair{\nabla f(w_k)}{w_{k + 1} - w_k} \nn \\
    &=& -\frac{1}{2L} \dotpair{\nabla f(w_k)}{g_k} \nn \\
    &\leq& \frac{1}{2L} \norm{\nabla f(w_k)} \cdot\norm{g_k}. \nn
}
The quantity $\norm{\nabla f(w_k)}$ is at most $\norm{\nabla f(w_k) - g_k} + \norm{g_k} \leq \eps + 4\eps = 5\eps$. Thus, the above yields
\myeqn{
    f(w_{k + 1}) - f(w_k)
    \leq \frac{1}{2L} \cdot 5 \eps \cdot 4 \eps = \frac{10 \eps^2}{L}. \label{eqn:simple-alg:h3}
}
Plugging \eqref{eqn:simple-alg:h2} and \eqref{eqn:simple-alg:h3} into \eqref{eqn:simple-alg:h1} gives
\myeqn{
    f(w_K) - f(w_{K'}) &\leq& \sum_{k \in S_<} f(w_{k + 1}) - f(w_k) \nn \\
    &\leq& |S_<| \cdot \fr{10 \eps^2}{L} 
    \leq K \cdot \fr{10 \eps^2}{L} \nn \\
    &\leq& \frac{LR}{4\eps} \cdot \fr{10 \eps^2}{L} = \frac{5 \eps R}{2}. \nn
}
Combining the above with \eqref{eqn:simple-alg:h4} yields
\myeqn{
    f(w_K) - f(w^*) &\leq& 5 \eps \norm{w^*} + \frac{5 \eps R}{2} = O(\eps R). \nn
}

\section{Data Preprocessing for Experiments} \label{app:exp}

This section explains the steps taken to preprocess the three datasets used in our experiments after they were downloaded from the LIBSVM website.

\vgap

The ijcnn1 collection from LIBSVM includes 49,990 training pairs and 91,701 testing pairs. We combined these into a single dataset with a total cardinality of 141,691. Each pair $(x, y)$ has a 22-dimensional vector $x$ and a label $y \in \set{-1, 1}$. We mapped the labels to $\set{0, 1}$ by modifying every $y = -1$ to $y = 0$. Furthermore, we appended a constant coordinate of 1 to each $x$, resulting in a final dimensionality of 23. This is a standard procedure to incorporate a bias term into the dot product $\dotpair{w}{x}$ within the loss function.

\vgap

The covtype dataset (specifically, the covtype.binary.scaled version) includes 581,012 pairs. Each pair $(x, y)$ has a 54-dimensional vector $x$ and a label $y \in \set{1, 2}$. We used all the pairs, but mapped the labels to $\set{0, 1}$ by subtracting 1 from every $y$. We also appended a constant coordinate of 1 to each $x$, resulting in 55-dimensional vectors.

\vgap

The HIGGS dataset includes 11,000,000 pairs. Each pair $(x, y)$ has a 28-dimensional vector $x$ and a label $y \in \set{0, 1}$. We utilized the entire dataset but excluded dimensions ``9'', ``13'', ``17'', ``21'' because around half of the pairs have missing values on those dimensions. Finally, we also appended a constant coordinate of 1 to each $x$, bringing the total dimensionality to 25.

\end{document}